\documentclass[10pt,twocolumn]{article}

\usepackage[a4paper,margin=0.75in,columnsep=0.25in]{geometry}

\usepackage[T1]{fontenc}
\usepackage{lmodern}
\usepackage{microtype}
\usepackage{amsmath}
\usepackage{amssymb}
\usepackage{amsthm}
\usepackage{bm}
\usepackage{dsfont}
\usepackage{xspace}
\usepackage{graphicx}
\usepackage{caption}
\usepackage{subcaption}
\usepackage{booktabs}
\usepackage{algorithm}
\usepackage[noend]{algpseudocode}
\usepackage{newfloat}
\usepackage{listings}

\DeclareCaptionStyle{ruled}{
  labelfont=normalfont,
  labelsep=colon,
  strut=off
}

\floatstyle{ruled}
\newfloat{listing}{tb}{lst}{}
\floatname{listing}{Listing}

\usepackage[hyphens]{url}
\usepackage{natbib}
\usepackage[hidelinks]{hyperref}
\usepackage[nameinlink,noabbrev]{cleveref}
\usepackage{enumitem}
\newcommand{\our}{\emph{HE-CV-UCB}\xspace}

\newtheorem{theorem}{Theorem}
\newtheorem{lemma}[theorem]{Lemma}

\theoremstyle{definition}

\theoremstyle{remark}
\newtheorem{remark}[theorem]{Remark}

\title{Bayesian Optimization of CVaR with Unknown Hyperparameters}

\author{%
    Valentin Margraf$^{1,2}$,
    Jonas Hanselle$^{1,2,3}$,
    Viktor Bengs$^{3}$,\\
    Serafima Lebedeva$^{3,6}$,
    Eyke H\"ullermeier$^{1,2,3}$,
    Marcel Wever$^{4,5}$\\ \\
    \small
    $^{1}$LMU Munich\\
    \small$^{2}$Munich Center for Machine Learning (MCML)\\
    \small$^{3}$German Research Center for Artificial Intelligence (DFKI)\\
    \small$^{4}$Leibniz University Hannover\\
    \small$^{5}$L3S Research Center\\
    \small$^{6}$RPTU Kaiserslautern-Landau
}

\date{}
\begin{document}

\maketitle

\begin{abstract}
We study Bayesian optimization of the Conditional Value-at-Risk (CVaR) of an expensive black-box function that depends on a decision and an environmental variable. Existing no-regret algorithms assume that the correct Gaussian process (GP) kernel and its hyperparameters are known. In practice, however, they must be estimated from limited data, often yielding a misspecified surrogate. This is especially damaging for CVaR because misspecification can distort the estimated tail-risk and thus lead to catastrophic decisions. We propose HE-CV-UCB, the first no-regret algorithm for CVaR Bayesian optimization with unknown GP hyperparameters. Given a finite candidate set containing a well-specified model, HE-CV-UCB adaptively eliminates those inconsistent with the observed data, actively queries where the survivors disagree, and returns a design that is robust across the surviving candidates. We prove that it attains cumulative regret matching the oracle rate of CV-UCB up to a factor depending only on the candidate set's size. Across synthetic benchmarks and three real-world applications, HE-CV-UCB achieves competitive performance and frequently outperforms existing CVaR Bayesian optimization methods.
\end{abstract}

\section{Introduction}
\label{sec:intro}

Bayesian optimization (BO) is a popular framework for optimizing expensive black-box functions~\cite{garnett-book23a}, with applications spanning materials design~\cite{zhang-sr20a}, drug discovery~\cite{colliandre-mmb24a}, and hyperparameter optimization~\cite{snoek-nips12a,lindauer-jmlr22a}. BO proceeds iteratively by fitting a probabilistic surrogate model to the evaluations observed so far, typically a Gaussian process (GP)~\cite{rasmussen-book06a}, and selecting new evaluations through an acquisition function that balances exploration and exploitation. In many real-world applications, however, the outcome depends not only on the controllable decision variable but also on uncertain environmental conditions. A portfolio allocation is exposed to stochastic market movements~\cite{liang-jim17a,millar-gecco25a}, a treatment plan to variability in patient response~\cite{yu-medphys17a}, and an solver configuration to differences in problem-instance difficulty~\cite{hutter-lion14a}. 
Classical BO typically optimizes expected performance. In safety-critical applications, however, optimizing the average can be misleading: a decision with excellent expected performance may still incur catastrophic losses under unfavorable conditions. A portfolio may collapse during a market downturn, a treatment may fail for a vulnerable patient, or a solver configuration may time out on particularly difficult instances. This motivates risk-sensitive objectives such as Conditional Value-at-Risk (CVaR), which measures the expected outcome among the worst $(1-\alpha)$-fraction of realizations for a risk level $\alpha\in(0,1)$.

\begin{figure}[t]
    \includegraphics[width=0.98\linewidth]{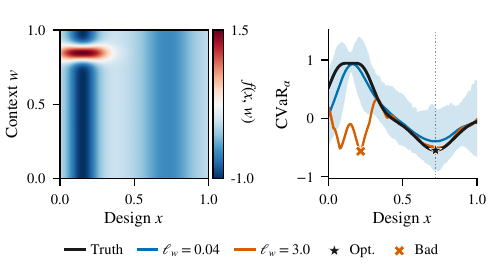}
\caption{Surrogate misspecification distorts CVaR estimation. \textbf{Left:} the objective $f(x,w)$, with a sharp ridge at $w\approx0.85$ for $x\lesssim0.4$. \textbf{Right:} the true $\mathrm{CVaR}_\alpha[f(x,W)]$ over designs $x$ (black). A well-specified surrogate ($\ell_w=0.04$, blue) resolves the ridge and its CVaR band contains the truth, while the over-smoothed surrogate ($\ell_w=3.0$, orange) is tail-blind and recommends a near-worst design ($x \approx 0.22$).}
  \label{fig:fig1}
\end{figure}

Existing CVaR BO algorithms, such as CV-UCB and CV-TS~\citep{nguyenquoc-neurips21a}, achieve sublinear regret under the assumption that the Gaussian process is correctly specified. In practice, this assumption is rarely satisfied. Kernel families are chosen heuristically, and their hyperparameters are typically selected from limited observations by maximum-likelihood or maximum-a-posteriori estimation, procedures that can be unreliable in finite-data regimes~\citep{lalchand-aabi19a} and thus result in misspecified surrogates. Surrogate misspecification is particularly harmful for CVaR optimization because it can translate into inaccurate risk estimates and consequently unsafe decisions, as illustrated in Figure~\ref{fig:fig1}. While several algorithms address unknown hyperparameters in mean-based Bayesian optimization~\citep{berkenkamp-jmlr19a,bogunovic-neurips21a,ziomek-neurips24a,ziomek-aistats25a}, the corresponding problem for CVaR optimization has remained unexplored. 

We close this gap with Hyperparameter-Elimination CVaR-UCB (\our), the first no-regret algorithm for CVaR Bayesian optimization that does not require the true kernel hyperparameters to be known. Instead, \our assumes only that the true model belongs to a finite candidate set. Building on the hyperparameter-elimination framework of \citet{ziomek-aistats25a}, our algorithm maintains this set of candidates, discards those whose predictions become statistically inconsistent with the observed data, actively queries where the surviving candidates disagree, and recommends a design that is robust across the remaining candidates. Such candidate sets can be obtained for example from surrogates fitted on related tasks, hyperparameter posterior samples, or domain knowledge. Our \textbf{contributions} are as follows.

%\paragraph{Contributions.}
\begin{itemize}
    \item We identify surrogate misspecification as a key challenge for CVaR Bayesian optimization because it can lead to inaccurate risk estimates and thus catastrophic decisions.
\item We propose \our, the first no-regret algorithm for CVaR Bayesian optimization with unknown kernel hyperparameters, which maintains a finite set of candidate models, actively eliminates those inconsistent with the observations, and recommends a design that is robust across the surviving candidates.
    \item We prove that \our achieves cumulative regret
    $
    \mathcal{O}\!\left(|\Theta|B + |\Theta|\sqrt{\xi_TT} + \beta_T|\Theta|\sqrt{T\gamma_T}\right),
    $
    recovering the oracle rate of CV-UCB up to a multiplicative factor of $|\Theta|$.
    \item We demonstrate on synthetic benchmarks and three real-world applications that \our achieves statistically significant improvements over competing methods while never being significantly outperformed.
\end{itemize}

\section{Related Work}

\paragraph{Risk-averse Bayesian optimization.} 
Bayesian optimization with environmental uncertainty has been studied under different objectives, ranging from expectation-based criteria to robust formulations that optimize against worst-case realizations~\citep{bogunovic-neurips18a,kirschner-aistats20a}. Risk measures such as Value-at-Risk (VaR) and Conditional Value-at-Risk provide an intermediate approach by focusing on adverse outcomes in the tail of the induced outcome distribution. \citet{cakmak-nips20a} proposed acquisition functions for general risk measures, including CVaR, but without regret guarantees. On the theoretical side, \citet{nguyenquoc-icml21a} studied Bayesian optimization of VaR, and \citet{nguyenquoc-neurips21a} introduced CV-UCB and CV-TS, the first algorithms with sublinear cumulative regret for CVaR. These guarantees, however, rely on a correctly specified GP surrogate with known kernel hyperparameters. 

\paragraph{Unknown kernel hyperparameters.}
Several works address unknown kernel hyperparameters and model misspecification in expectation-based Bayesian optimization. Practical approaches marginalize over hyperparameter posteriors~\citep{snoek-nips12a,lalchand-aabi19a}, but typically lack regret guarantees. On the theoretical side, \citet{bogunovic-neurips21a} analyze regret under bounded misspecification. \citet{berkenkamp-jmlr19a} adapt the kernel lengthscale online to guarantee inclusion of the true function in the reproducing kernel Hilbert space (RKHS) of the selected kernel, though the schedule is data-independent and known to be over-exploratory. \citet{ziomek-neurips24a} maintain multiple GP bandit learners with different lengthscales and eliminate learners whose performance becomes statistically inferior. More recently, \citet{ziomek-aistats25a} propose PE-GP-UCB for the time-varying setting, which handles uncertainty over a finite collection of GP surrogates, including different kernels, hyperparameters, and priors. \citet{sandberg-arxiv26a} extend their framework to Thompson sampling. All of these approaches, however, focus on expectation-based objectives.

Hyperparameter uncertainty in risk-averse Bayesian optimization, however, remains unexplored. To the best of our knowledge, \our is the first no-regret algorithm for CVaR Bayesian optimization that does not assume the true kernel hyperparameters to be known. We build on the elimination framework of \citet{ziomek-aistats25a} and extend it to CVaR optimization, where uncertainty over the surrogate affects a tail-risk functional rather than the objective value directly.

\section{CVaR Bayesian Optimization}
\label{sec:background}
We consider an unknown objective function $f : \mathcal{X} \times \mathcal{W} \to \mathbb{R}$, where $\mathcal{X} \subset \mathbb{R}^d$ is the design space and $\mathcal{W} \subset \mathbb{R}^{d_W}$ is the support of an environmental random variable $W$ with distribution $\mathbb{P}_W$. We assume $\mathbb{P}_W$ is known or can be estimated, and represent it by a finite set of environmental scenarios on which all CVaR quantities are evaluated. We treat large values of $f$ as undesirable, e.g., loss, runtime, or prediction error, and seek a design $x$ that minimizes
\begin{equation}\label{eq:obj}
x^\star \in \arg\min_{x \in \mathcal{X}} \mathrm{CVaR}_\alpha[f(x,W)].
\end{equation}
The risk measure is taken with respect to $W \sim \mathbb{P}_W$ for a fixed design $x$. Formally, for a real-valued random variable $Z$ with $\mathbb{E}|Z| < \infty$ and risk level $\alpha \in (0,1)$, the Value-at-Risk is the $\alpha$-quantile $\mathrm{VaR}_\alpha[Z] := \inf\{ z \in \mathbb{R} : \mathbb{P}(Z \le z) \ge \alpha \}$, and the Conditional Value-at-Risk is the average over the worst $(1-\alpha)$-fraction of outcomes $\mathrm{CVaR}_\alpha[Z] := \frac{1}{1-\alpha}\int_\alpha^1 \mathrm{VaR}_{\alpha'}[Z]\,\mathrm{d}\alpha'$.
%\begin{equation}
% \mathrm{CVaR}_\alpha[Z] := \frac{1}{1-\alpha}\int_\alpha^1 5\mathrm{VaR}_{\alpha'}[Z]\,\mathrm{d}\alpha'.
% \end{equation}
Larger $\alpha$ corresponds to greater risk aversion, with $\alpha \to 0$ recovering the expectation $\mathbb{E}[Z]$.

During optimization, the algorithm does not observe random realizations of $W$. Instead, at iteration $t$ it actively selects a query $(x_t,w_t) \in \mathcal{X} \times \mathcal{W}$ and observes a noisy evaluation
\begin{equation}
y_t = f(x_t,w_t) + \varepsilon_t, \qquad \varepsilon_t \sim \mathcal{N}(0,\sigma_\varepsilon^2).
\end{equation}

After optimization, only the recommended design is deployed, while the realized environment is drawn according to $W \sim \mathbb{P}_W$. We measure performance by the cumulative regret $R_T := \sum_{t=1}^T r_t$ with per-round regret $r_t := \mathrm{CVaR}_\alpha[f(x_t,W)] - \mathrm{CVaR}_\alpha[f(x^\star,W)]$. An algorithm is called no-regret if $R_T/T\rightarrow0$.

We model $f$ using a Gaussian process surrogate indexed by $\theta$, which specifies the kernel $k_\theta$ and its hyperparameters. Given observations $\mathcal{D}_{t-1} = \{(x_i,w_i,y_i)\}_{i=1}^{t-1}$, the GP posterior provides a predictive mean $\mu_{t-1}^{\theta}(x,w)$ and standard deviation $\sigma_{t-1}^{\theta}(x,w)$. We work in the frequentist setting and treat $f$ as a fixed unknown function. A surrogate $\theta$ is well specified if $f$ belongs to the RKHS induced by $k_\theta$ and satisfies $\|f\|_{\mathcal{H}_{k_\theta}}\le B$. Under this assumption, standard GP confidence bounds imply that, for confidence level $\delta\in(0,1)$, the choice
$\beta_t^\theta =
B+\sigma_\varepsilon(2(\gamma_{t-1}^{\theta}+1+\ln(1/\delta))^{1/2}$
yields pointwise confidence intervals
$l_{t-1}^{\theta}(x,w)\leq f(x,w)\leq u_{t-1}^{\theta}(x,w)$
for all $(x,w)\in\mathcal{X}\times\mathcal{W}$ and all $t\geq1$ with probability at least $1-\delta$~\citep{chowdhury-icml17a}, where
$u_{t-1}^{\theta}(x,w)=\mu_{t-1}^{\theta}(x,w)+\beta_t^\theta\sigma_{t-1}^{\theta}(x,w)$
and
$l_{t-1}^{\theta}(x,w)=\mu_{t-1}^{\theta}(x,w)-\beta_t^\theta\sigma_{t-1}^{\theta}(x,w)$.
Here,
$\gamma_t^\theta :=
\max_{A\subset\mathcal{X}\times\mathcal{W},\,|A|=t}
\frac12\log\det(I+\sigma_\varepsilon^{-2}K_A^\theta)$
is the maximum information gain associated with kernel $k_\theta$, where $K_A^\theta$ is the kernel matrix induced by $k_\theta$ on any set $A$ of $t$ query points in $\mathcal{X}\times\mathcal{W}$. As standard in GP bandit analyses, we further assume that $k_\theta((x,w),(x,w))\leq1$.

CV-UCB~\citep{nguyenquoc-neurips21a} lifts the pointwise bands to CVaR bounds $c_u(x;\alpha)=\mathrm{CVaR}_\alpha[u(x,W)]$ and $c_l(x;\alpha)=\mathrm{CVaR}_\alpha[l(x,W)]$, which by monotonicity of CVaR satisfy $c_l(x;\alpha) \le \mathrm{CVaR}_\alpha[f(x,W)] \le c_u(x;\alpha)$. CV-UCB then selects the design optimistically, $x_t \in \arg\min_{x \in \mathcal{X}} c_{l}(x;\alpha)$, and queries the environment at a \emph{lacing value}, a context whose pointwise band gap upper-bounds the CVaR band gap at the selected design. For a single surrogate this is attained at $w_t \in \arg\max_{w \in \mathcal{W}} \sigma_{t-1}(x_t,w)$, which maximizes $u_{t-1}-l_{t-1} = 2\beta_t\,\sigma_{t-1}$. If the surrogate is well-specified, CV-UCB achieves sublinear regret.\footnote{\citet{nguyenquoc-neurips21a} analyze CV-UCB in the Bayesian setting, treating $f$ as a GP sample with $\mathcal{X}$ discrete; while their companion VaR algorithm~\citep{nguyenquoc-icml21a} uses the same RKHS bound we adopt. Their settings differ only in the admissible $\beta_t$: the lifting to CVaR uses nothing beyond $l \le f \le u$ almost surely and monotonicity of CVaR.} 

In practice, however, the kernel and its hyperparameters are unknown and must be estimated from finite data. Consequently, the confidence guarantees above may fail, and a misspecified surrogate can produce inaccurate CVaR estimates and lead to poor decisions (Figure~\ref{fig:fig1}). 

\section{CVaR BO with Unknown Hyperparameters}
\label{sec:method}

We now develop \our, a no-regret algorithm for CVaR Bayesian optimization that does not require a well-specified surrogate to be known. Instead, we assume a finite candidate set $\Theta$ of GP models containing the unknown correct candidate $\theta^\star$, and extend the elimination framework of \citet{ziomek-aistats25a} to the CVaR setting. Candidates are eliminated based on statistical inconsistency with the observed data, while the acquisition rule reasons over the surviving candidates to construct valid CVaR bounds. Since the optimistic and pessimistic bounds may originate from different candidates, the resulting procedure requires a two-sided elimination mechanism and a generalized lacing argument.

\subsection{Proposed algorithm}
\label{sec:he-cv-setup}

Each GP surrogate $\theta \in \Theta$ provides posterior mean $\mu_{t-1}^\theta(x,w)$ and standard deviation $\sigma_{t-1}^\theta(x,w)$. We define the corresponding pointwise confidence bounds as $u_{t-1}^\theta := \mu_{t-1}^\theta + \beta_t^\theta \sigma_{t-1}^\theta$ and $l_{t-1}^\theta := \mu_{t-1}^\theta - \beta_t^\theta \sigma_{t-1}^\theta$. For the correct candidate $\theta^\star$, these bounds contain the unknown objective with high probability, i.e., $l_{t-1}^{\theta^\star}(x,w) \leq f(x,w) \leq u_{t-1}^{\theta^\star}(x,w)$. Since CVaR is monotone, this implies valid CVaR confidence bounds $c_{l_{t-1}}^{\theta^\star}(x;\alpha) \leq \mathrm{CVaR}_\alpha[f(x,W)] \leq c_{u_{t-1}}^{\theta^\star}(x;\alpha)$, with $c_{u_{t-1}}^\theta(x;\alpha):=\mathrm{CVaR}_\alpha[u_{t-1}^\theta(x,W)]$ and $c_{l_{t-1}}^\theta(x;\alpha):=\mathrm{CVaR}_\alpha[l_{t-1}^\theta(x,W)]$. Let $\Theta_{t-1}\subseteq\Theta$ denote the candidates surviving at the beginning of iteration $t$.

\paragraph{Doubly-optimistic decision.}
At each iteration, the algorithm selects a design and a candidate jointly by minimizing the lower CVaR confidence bound,
\begin{equation}
\label{eq:he-cv-xt}
(x_t, \theta_t) \;\in\; \arg\min_{(x,\theta) \in \mathcal{X} \times \Theta_{t-1}} c_{l_{t-1}}^\theta(x;\alpha).
\end{equation}
This choice is optimistic in two respects: it searches for the most promising design and simultaneously trusts the most optimistic surviving surrogate (the candidate that promises the lowest CVaR). Because $\theta^\star\in\Theta_{t-1}$ as long as the correct candidate is not eliminated, the selected value provides a lower bound on the unknown optimum: $c_{l_{t-1}}^{\theta_t}(x_t;\alpha)
\leq
c_{l_{t-1}}^{\theta^\star}(x^\star;\alpha)
\leq
\mathrm{CVaR}_\alpha[f(x^\star,W)].$

To obtain an upper bound on the CVaR of the selected design, we select the most pessimistic surviving candidate at $x_t$,
\begin{equation}
\label{eq:he-cv-vt}
\vartheta_t \;\in\; \arg\max_{\theta \in \Theta_{t-1}} c_{u_{t-1}}^\theta(x_t;\alpha).
\end{equation}
Again, because $\theta^\star$ remains feasible, $c_{u_{t-1}}^{\vartheta_t}(x_t;\alpha)
\geq
c_{u_{t-1}}^{\theta^\star}(x_t;\alpha)
\geq
\mathrm{CVaR}_\alpha[f(x_t,W)].$ The two selected candidates need not coincide: $\theta_t$ represents the most optimistic explanation of the optimum, whereas $\vartheta_t$ represents the most conservative explanation of the selected design. Combining the two bounds yields a bound on the per-round regret
\begin{equation}
\label{eq:certified-gap}
r_t
\leq
c_{u_{t-1}}^{\vartheta_t}(x_t;\alpha)
-
c_{l_{t-1}}^{\theta_t}(x_t;\alpha).
\end{equation}

\paragraph{Lacing value.}
The right-hand side of \eqref{eq:certified-gap} is a functional over the full environmental distribution, while only one environmental realization can be queried per iteration. Following the terminology of CV-UCB~\citep{nguyenquoc-neurips21a}, we call a query point a \emph{lacing value} if the pointwise confidence gap at that point upper-bounds the corresponding CVaR gap. Since our optimistic and pessimistic bounds may come from different surviving surrogates, we select
\begin{equation}
\label{eq:he-cv-lv}
w_t \;\in\; \arg\max_{w \in \mathcal{W}} \underbrace{\Bigl[\, u_{t-1}^{\vartheta_t}(x_t,w) - l_{t-1}^{\theta_t}(x_t,w) \,\Bigr]}_{=:\; G_t(w)}.
\end{equation}
Unlike standard CV-UCB, the two functions defining the gap need not be ordered pointwise, so $G_t(w)$ may be negative. Nevertheless, Lemma~\ref{lem:cvar-shift} gives $r_t \leq G_t(w_t)$, showing that the selected query controls the instantaneous regret. Hence, $w_t$ is a lacing value for the cross-surrogate confidence gap. The resulting gap captures both posterior uncertainty within each candidate and disagreement between candidates:
\begin{align*}
u_{t-1}^{\vartheta_t}(x_t,w) - l_{t-1}^{\theta_t}(x_t,w) = \underbrace{\mu_{t-1}^{\vartheta_t}(x_t,w) - \mu_{t-1}^{\theta_t}(x_t,w)}_{\text{mean disagreement}} \\ + {\beta_t^{\vartheta_t}}\,\sigma_{t-1}^{\vartheta_t}(x_t,w) + {\beta_t^{\theta_t}}\,\sigma_{t-1}^{\theta_t}(x_t,w),
\end{align*}
When $\theta_t=\vartheta_t$, the disagreement term vanishes and the gap reduces to the standard GP confidence width. Otherwise, the algorithm actively queries regions where surviving candidates disagree.

\paragraph{Elimination criteria.}
Since the candidate set $\Theta$ may contain misspecified surrogates, we eliminate candidates whose predictions become statistically incompatible with the observations. Importantly, the analysis does not require identifying the true surrogate $\theta^\star$ uniquely: Theorem~\ref{thm:regret} only requires that at least one correctly specified candidate survives with high probability. Additional surviving candidates may increase conservatism but cannot invalidate the guarantee.

At each iteration, we test the two candidates that define the current CVaR bounds: the optimistic candidate $\theta_t$ and the pessimistic candidate $\vartheta_t$. Their residuals are
$\eta_t := y_t-\mu_{t-1}^{\theta_t}(x_t,w_t)$
and
$\hat{\eta}_t := y_t-\mu_{t-1}^{\vartheta_t}(x_t,w_t)$.
For a candidate $\theta$, let
$S^\theta_t=\{i\leq t:\theta_i=\theta\}$ and
$\hat S^\theta_t=\{i\leq t:\vartheta_i=\theta\}$
denote the iterations on which $\theta$ contributed to the optimistic or pessimistic bound, respectively. 

For a candidate $\theta$ that is selected as either an optimistic or pessimistic bound, the corresponding residuals decompose into observation noise and model error: $\eta_i=\varepsilon_i+(f(x_i,w_i)-\mu_{i-1}^{\theta}(x_i,w_i))$.
For a correctly specified candidate, the model error is controlled by the GP confidence bound,
$|f(x_i,w_i)-\mu_{i-1}^{\theta}(x_i,w_i)|
\leq
\beta_i^\theta\sigma_{i-1}^{\theta}(x_i,w_i)$,
while the accumulated observation noise is bounded by the following concentration result, which adapts Lemma~5.1 of \citet{ziomek-aistats25a} to our setting.

\begin{lemma}[Noise concentration]
\label{lem:noise-concentration}
Let $\delta_B \in (0,1)$ and set
$\xi_t = 2\sigma_\varepsilon^2\log(2|\Theta|\pi^2t^2/3\delta_B)$.
Then with probability at least $1-\delta_B$, for all $\theta\in\Theta$, all $t\leq T$, and both
$S\in\{S^\theta_t,\hat S^\theta_t\}$,
it holds
$\left|\sum_{i\in S}\varepsilon_i\right|\leq\sqrt{\xi_t|S|}$.
\end{lemma}

Combining the noise and model-error bounds gives a threshold that a correctly specified candidate cannot exceed with high probability. We therefore remove the optimistic candidate $\theta_t$ whenever

\begin{equation}
\label{eq:he-cv-elim-u}
\Bigl|\,\sum\limits_{i \in S^{\theta_t}_t} \eta_i\,\Bigr|
\;>\;
\sqrt{\xi_t\,|S^{\theta_t}_t|}%_{\text{noise budget}}
+
\sum\limits_{i \in S^{\theta_t}_t}
\beta_i^{\theta_t}\,\sigma_{i-1}^{\theta_t}(x_i,w_i)%_{\text{model-error budget}},
\end{equation}
and, symmetrically, remove the pessimistic candidate $\vartheta_t$ if
\begin{equation}
\label{eq:he-cv-elim-l}
\Bigl|\,\sum\limits_{i \in \hat S^{\vartheta_t}_t} \hat\eta_i\,\Bigr|
\;>\;
\sqrt{\xi_t\,|\hat S^{\vartheta_t}_t|}
+
\sum\limits_{i \in \hat S^{\vartheta_t}_t}
\beta_i^{\vartheta_t}\,\sigma_{i-1}^{\vartheta_t}(x_i,w_i).
\end{equation}
With high probability, the correctly specified candidate $\theta^\star$ violates neither criterion and is therefore never eliminated. Conversely, a candidate that violates one of the criteria has accumulated residual errors that are incompatible with its own confidence model and the observed data. The two criteria are evaluated independently, so an iteration may eliminate the optimistic candidate, the pessimistic candidate, both, or neither; we provide the detailed procedure in Algorithm~\ref{alg:ouralg}.
\begin{algorithm}[t]
\caption{\our}
\label{alg:ouralg}
\begin{algorithmic}[1]
\Require Candidate set $\Theta$; risk level $\alpha$; confidence parameters
$\{\beta_t^\theta\}_{\theta\in\Theta}$ and $\{\xi_t\}$; horizon $T$;
initial data $\mathcal{D}_0$

\State $\Theta_0 \gets \Theta$;\quad
$S^\theta_0,\hat S^\theta_0 \gets \emptyset$ for all $\theta\in\Theta$

\For{$t=1,\ldots,T$}

    \State Fit a GP posterior for every $\theta\in\Theta_{t-1}$
    \State $(x_t,\theta_t)
    \in \arg\min\limits_{(x,\theta)\in\mathcal{X}\times\Theta_{t-1}}
    c_{l_{t-1}}^\theta(x;\alpha)$
    \Comment{\eqref{eq:he-cv-xt}}
    \State $\vartheta_t
    \in \arg\max\limits_{\theta\in\Theta_{t-1}}
    c_{u_{t-1}}^\theta(x_t;\alpha)$
    \Comment{\eqref{eq:he-cv-vt}}
    \State $w_t
    \in \arg\max\limits_{w\in\mathcal{W}}
    [
    u_{t-1}^{\vartheta_t}(x_t,w)
    -
    l_{t-1}^{\theta_t}(x_t,w)
    ]$
    \Comment{\eqref{eq:he-cv-lv}}
    \State Observe
    $y_t=f(x_t,w_t)+\varepsilon_t$
    \State $\mathcal{D}_t
    \gets
    \mathcal{D}_{t-1}\cup\{(x_t,w_t,y_t)\}$
    \State $\eta_t
    \gets
    y_t-\mu_{t-1}^{\theta_t}(x_t,w_t)$
    \State $\hat\eta_t
    \gets
    y_t-\mu_{t-1}^{\vartheta_t}(x_t,w_t)$
    \State Update
    $S^{\theta_t}_t\gets S^{\theta_t}_{t-1}\cup\{t\}$,
    $\hat S^{\vartheta_t}_t\gets \hat S^{\vartheta_t}_{t-1}\cup\{t\}$
    \State $\Theta_t\gets\Theta_{t-1}$
    \If{criterion \eqref{eq:he-cv-elim-u} holds}
        \State $\Theta_t\gets\Theta_t\setminus\{\theta_t\}$
    \EndIf
    \If{criterion \eqref{eq:he-cv-elim-l} holds}
        \State $\Theta_t\gets\Theta_t\setminus\{\vartheta_t\}$
    \EndIf
\EndFor
\State \Return $\hat{x} \in
\arg\min\limits_{x \in \{x_1,\ldots,x_T\}}
\max\limits_{\theta \in \Theta_T}\mathrm{CVaR}_\alpha[\mu_T^\theta(x, W)]$
\end{algorithmic}
\end{algorithm}

\subsection{Regret guarantee}
\label{sec:he-cv-regret}

We now show that the elimination and lacing mechanisms together provide a no-regret guarantee. The key step is to relate the CVaR regret incurred by the selected design to the cross-surrogate confidence gap at the environmental queries.

The gap $G_t(w)$ contains a mean-disagreement term and may therefore be negative, i.e., the optimistic and pessimistic bounds need not be pointwise ordered. We use the following CVaR difference inequality, which only requires an upper bound on the difference between the two random variables.

\begin{lemma}[CVaR shift]
\label{lem:cvar-shift}
For integrable $X, Y$ on a common probability space and $\alpha \in (0,1)$, $\mathrm{CVaR}_\alpha[X] - \mathrm{CVaR}_\alpha[Y] \le \operatorname*{ess\,sup}(X - Y)$.
\end{lemma}

Lemma~\ref{lem:cvar-shift} follows from the monotonicity and translation invariance of CVaR; the proof is provided in the Appendix. Applying the lemma with $X=u_{t-1}^{\vartheta_t}(x_t,W)$ and $Y=l_{t-1}^{\theta_t}(x_t,W)$ yields $r_t
\leq
\operatorname*{ess\,sup}_{w\in\mathcal W}
(
u_{t-1}^{\vartheta_t}(x_t,w)
-
l_{t-1}^{\theta_t}(x_t,w)
)
=
G_t(w_t),$
where the last equality follows from the definition of the lacing value. Hence, the selected environmental query controls the instantaneous regret, which enables the following cumulative regret bound (the full proof is given in the Appendix).

\begin{theorem}[Cumulative regret of \our]
\label{thm:regret}
Assume $f \in \mathcal{H}_{k_{\theta^\star}}$ with $\|f\|_{\mathcal{H}_{k_{\theta^\star}}} \le B$ for some $\theta^\star \in \Theta$, $\varepsilon_t \sim \mathcal{N}(0,\sigma_\varepsilon^2)$ i.i.d., and $k_\theta \le 1$ for all $\theta \in \Theta$. Set $\beta_t^\theta = B + \sigma_\varepsilon(2(\gamma_{t-1}^\theta + 1 + \ln(2/\delta)))^{1/2}$ and $\xi_t = 2\sigma_\varepsilon^2 \log(4|\Theta|\pi^2 t^2/3\delta)$. Then with probability at least $1 - \delta$, the cumulative regret satisfies
\begin{equation*}
\label{eq:he-cv-regret}
R_T \;\le\; \mathcal{O}\!\bigl(\,|\Theta|\,B + |\Theta|\,\sqrt{\xi_T\,T} + \beta_T|\Theta|\,\sqrt{T\,\,\gamma_T}\,\bigr),
\end{equation*}
with $\beta_T := \max_\theta \beta_T^\theta$ and $\gamma_T := \max_\theta \gamma_T^\theta$.
\end{theorem}

For kernels satisfying $\gamma_T=o(\sqrt{T})$ and a fixed finite candidate set, the bound is sublinear in $T$. Thus, \our achieves no-regret learning without knowing the true kernel hyperparameters. The factor $|\Theta|$ arises from the two-sided elimination procedure: unlike single-criterion elimination, the optimistic and pessimistic residual sequences can be interrupted by eliminations of the opposite candidate, requiring an additional charge for such gaps in the residual history.

Note that the guarantee concerns the sequence of queried designs $\{x_t\}_{t=1}^{T}$. Sublinear cumulative regret implies that the best queried design satisfies $\min_{t\leq T} r_t \leq R_T/T \rightarrow 0,$
i.e., at least one evaluated design achieves vanishing instantaneous regret.

\paragraph{Robust recommendation.}
After $T$ iterations the algorithm must commit to a single design. We return the
\emph{$\Gamma$-maximin} design over the surviving candidate set, i.e., the design
that minimizes the worst-case posterior mean CVaR across the remaining models~\cite{seidenfeld-synthese04a}:
\begin{equation}
\hat{x}\in
\arg\min_{x\in\{x_1,\ldots,x_T\}}
\max_{\theta\in\Theta_T}
\mathrm{CVaR}_\alpha[\mu_T^\theta(x,W)] .
\end{equation}

Thus, unlike the acquisition rule, which selects the most optimistic candidate, the final recommendation is conservative and accounts for remaining model uncertainty. It selects the explored design with the smallest worst-case posterior mean CVaR among all surviving candidates.

\section{Experiments}\label{sec:experiments}
  
We empirically evaluate our algorithm across a range of synthetic and real-world benchmarks to assess its ability to handle surrogate uncertainty in CVaR Bayesian optimization. In particular, we conduct experiments to answer the following research questions:
\begin{enumerate}[label=\textbf{RQ\arabic*.}, leftmargin=*, itemsep=2pt, topsep=4pt, parsep=0pt]
    \item If the assumptions of Theorem~\ref{thm:regret} hold, does \our recover the performance of an oracle CV-UCB using the correct surrogate?
    \item When candidates fit the bulk well but oversmooth the tail, can \our identify and overcome this failure mode of likelihood-based selection and averaging?
    \item Does \our remain effective on real-world transfer tasks, where we can not guarantee that the correct model is in the pool?
    \item How does \our compare to fitting a single surrogate on the target by MLE without transfer?
\end{enumerate}

To answer RQ1 and RQ2, we construct $\Theta$ synthetically from GPs with different lengthscales over $w$. To answer RQ3 and RQ4, we consider real-world optimization tasks in a transfer BO setting. In practice, many optimization problems are solved repeatedly on related tasks, so previously collected observations are often available and can be used to fit surrogate models offline before transferring them to a new target task~\citep{feurer-arxiv18a}. We therefore fit one surrogate per source task and transfer the resulting models to a held-out target. The transferred surrogates naturally form the candidate set $\Theta$. Unlike the synthetic construction, we here cannot guarantee that a correctly specified surrogate is in the pool, so the real-world experiments test \our beyond the assumptions of Theorem~\ref{thm:regret}.

For RQ2-RQ3, all compared methods receive the same transferred candidate pool. We compare against the following CVaR-BO baselines: CV-UCB, CV-TS~\citep{nguyenquoc-neurips21a}, and $\rho\text{-KG}^{\mathrm{apx}}$~\citep{cakmak-nips20a}. To ensure a fair comparison in the transfer setting, each of theses baseline selects from the transferred pool at every iteration the surrogate with the highest marginal likelihood on the target and then applies its respective acquisition function. As an additional baseline, we consider BMA-UCB, which aggregates the transferred surrogates by marginal-likelihood weighting and applies CV-UCB to the resulting predictive mixture. For RQ4, we study the value of the transferred candidate pool by comparing \our against single-task surrogates fitted from scratch on the target by MLE, which corresponds to the original setting of CV-UCB, CV-TS, and $\rho\text{-KG}^{\mathrm{apx}}$. This comparison isolates whether access to related source tasks provides an advantage over learning a surrogate solely from target observations. Shared hyperparameters ($\beta = 3$), initial designs, and GP conditioning are matched across methods, and we report the CVaR regret of the final recommended design. All methods are implemented in BoTorch~\citep{balandat-neurips20a}; full implementation details, additional results, and ablation studies are provided in the Appendix.

\begin{figure}[t]
\centering
\includegraphics[width=0.98\linewidth]{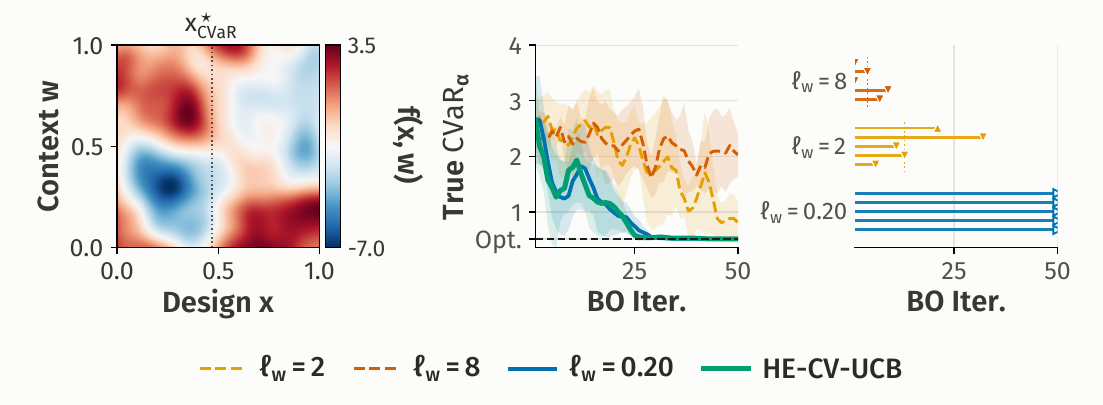}
\caption{Candidate pool of three GPs: the one generating the objective ($\ell_w = 0.20$, oracle) and two oversmoothers ($\ell_w \in \{2.0, 8.0\}$). \textbf{Left:} the GP sample $f(x,w)$. \textbf{Middle:} simple regret of the recommended design, averaged over 5 seeds. \textbf{Right:} \our eliminates both oversmoothers after a median of 4 and 13 iterations, respectively.}
    \label{fig:algo_validation}
\end{figure}

\subsection{Toy example (RQ1)}
To answer RQ1, we sample a two-dimensional test function from a GP with lengthscales $\ell_x = \ell_w = 0.2$. We provide \our with a pool of three candidates: the data-generating GP (oracle) and two misspecified GPs
with the same design-space lengthscale $\ell_x = 0.2$ but overly large environmental lengthscales ($\ell_w \in \{2.0, 8.0\}$). The pool thus contains the correctly specified data-generating GP, allowing us to investigate whether \our recovers oracle performance without knowing which
candidate generated the objective.

We run \our on this pool and, and compare against CV-UCB on each candidate separately: the run on the oracle is the performance ceiling, and the runs on the two oversmoothers show the failure mode. As Figure~\ref{fig:algo_validation} shows, \our eliminates both oversmoothed candidates while retaining the oracle across all seeds, and the simple regret of its recommended design converges to the oracle's. This confirms that \our can discard tail-blind surrogates and recover oracle performance when the data-generating GP is included in the candidate pool.

\subsection{Synthetic experiment (RQ2)}
To answer RQ2, we construct a synthetic objective that exposes the failure mode of likelihood-based model evaluation. We use the spike objective from the left panel of Figure~\ref{fig:fig1}, where a narrow region of the environmental space determines the CVaR-optimal design. The candidate pool consists of GPs whose lengthscales over $w$ bracket the true spike width, and the design-space lengthscale is fixed to $\ell_x = 0.1$ for all of them. Misspecification is thus confined to the dimension that drives the tail, which is precisely where marginal likelihood is blind: an oversmoothed surrogate fits the bulk well and attains high likelihood while missing the rare events that dominate the CVaR. The baselines CV-UCB, CV-TS, and $\rho$-KG$^{\mathrm{apx}}$ select by likelihood and BMA-UCB averages by likelihood weights, so all baselines inherit this blind spot.

As can be see in Figure~\ref{fig:exp_spike}, the baselines increasingly favor regions with poor CVaR performance, BMA-UCB plateaus at a CVaR regret around $1$, and \our reaches near-zero regret. The difference is in how \our chooses queries: rather than scoring global fit, it evaluates where the surviving candidates disagree, and since they already agree across the bulk, this disagreement concentrates in the tail. The lacing rule thus acquires exactly the observations that separate a spike-resolving surrogate from an oversmoothed one, which is why \our succeeds where likelihood-based selection and averaging fail, answering RQ2.

\begin{figure}[t]
    \centering
    \includegraphics[width=0.95\linewidth]{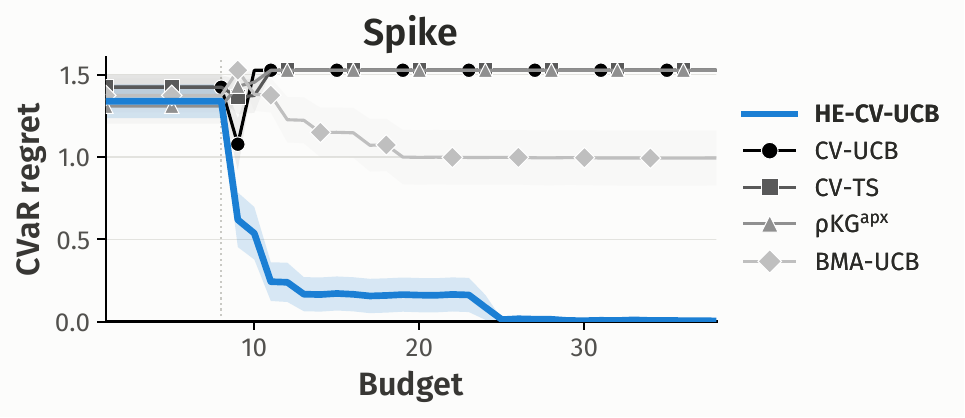}
    \caption{CVaR regret on the synthetic \textbf{spike function} averaged over 10 seeds. The baselines select or average high-likelihood source GPs that are tail-blind, \our removes the oversmoothers and reaches zero regret.}
    \label{fig:exp_spike}
\end{figure}

\begin{figure*}[t]
    \centering
    \includegraphics[width=.9\linewidth,trim={0 0.8cm 0 0.7cm},clip]{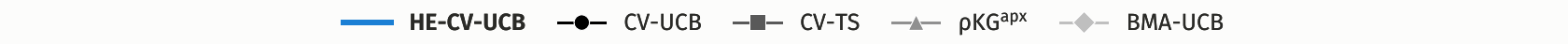}
\begin{subfigure}{0.32\linewidth}
    \centering
    \includegraphics[width=\linewidth]{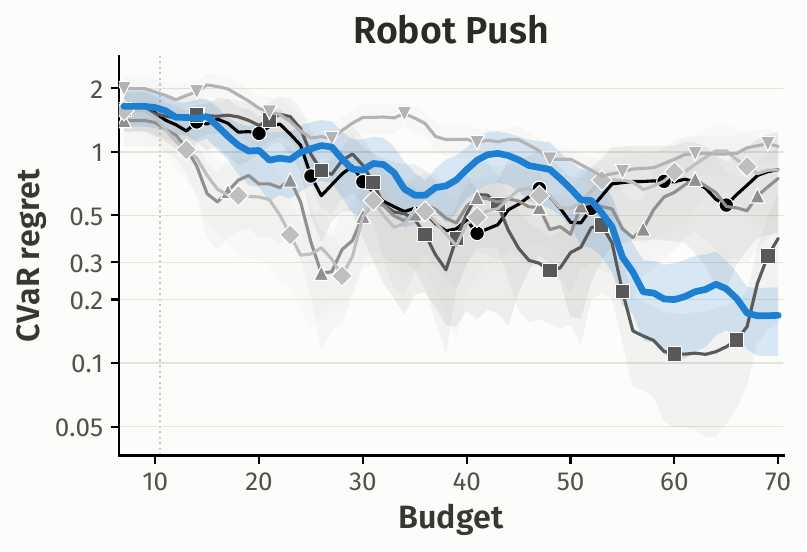}
\end{subfigure}
\hfill
\begin{subfigure}{0.32\linewidth}
    \centering
    \includegraphics[width=\linewidth]{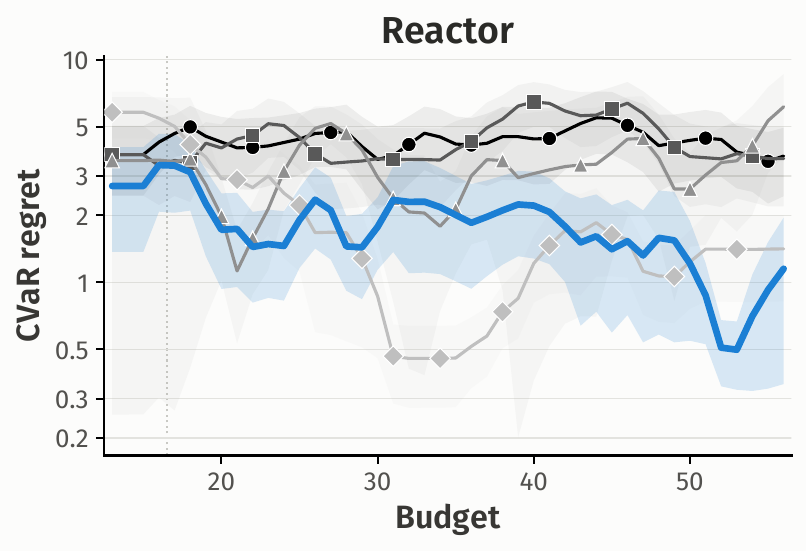}
\end{subfigure}
\hfill
\begin{subfigure}{0.32\linewidth}
    \centering
    \includegraphics[width=\linewidth]{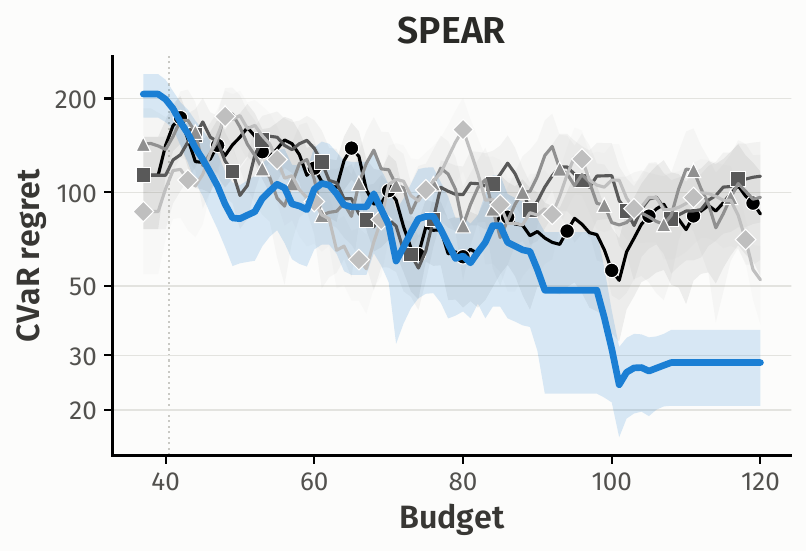}
\end{subfigure}
    \caption{CVaR regret of the recommended design vs.\ budget on robot pushing \textbf{(left)}, the exothermic batch reactor \textbf{(middle)}, and SAT solver configuration \textbf{(right)}. \our matches the best baselines on robot pushing and reaches substantially lower CVaR regret on the reactor and SAT. Mean $\pm$ 1 SEM over 10 seeds.}
    \label{fig:exp_realworld}
\end{figure*}

\subsection{Real-world experiments (RQ3)}

We now want to investigate whether \our stays effective on real world tasks. On each benchmark we fix a target task and derive its sources by varying a single interpretable factor of the problem, fitting one GP per source offline. The three benchmarks span the misspecification spectrum, from mild (a small goal rotation) to severe (a large change of scale).
 
\paragraph{Robot push.}
Taken from \citet{nguyenquoc-neurips21a}, this task has a robot pushing an object toward a fixed goal by choosing its position and push duration $x \in \mathcal{X}$ ($d_x = 3$). An uncontrollable perturbation $W$ ($d_W = 2$) displaces the robot's start location, and the loss is the object's residual distance to the goal. Small displacements barely deflect the object while large ones cause it to miss badly, motivating $\mathrm{CVaR}_\alpha$.

We construct $K = 8$ source tasks by displacing the fixed goal to different positions. The push dynamics are identical across all tasks, so a source surrogate captures the correct dynamics but places the optimum in the wrong location. This is the \emph{benign} end of the spectrum: every source models the target reasonably, so committing to the "wrong" one costs little, and the tail-blindness that hurts likelihood-based baselines elsewhere is not exposed here. As can be seen in Figure~\ref{fig:exp_realworld} (lef), CV-UCB, BMA-UCB and $\rho\text{-KG}^{\mathrm{apx}}$ perform similarly. \our matches the performance of the strongest baseline, CV-TS, while exhibiting a more consistent decrease in CVaR regret over the optimization horizon.

\begin{figure}[t]
\centering
    \includegraphics[width=\linewidth,trim={0.5cm 0.3cm 0.5cm 0.3cm},clip]{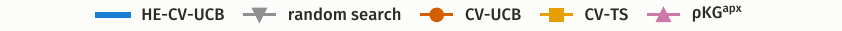}
\begin{subfigure}{0.32\linewidth}
    \centering
    \includegraphics[width=\linewidth]{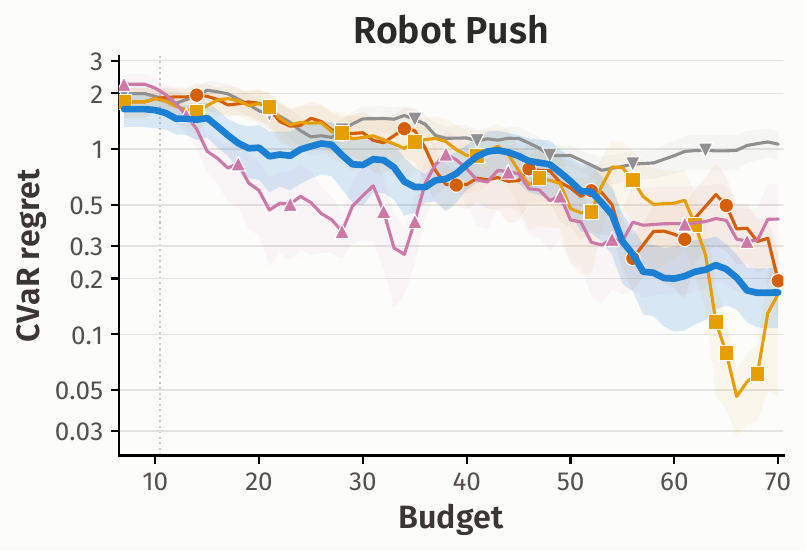}
\end{subfigure}
\hfill
\begin{subfigure}{0.32\linewidth}
    \centering
    \includegraphics[width=\linewidth]{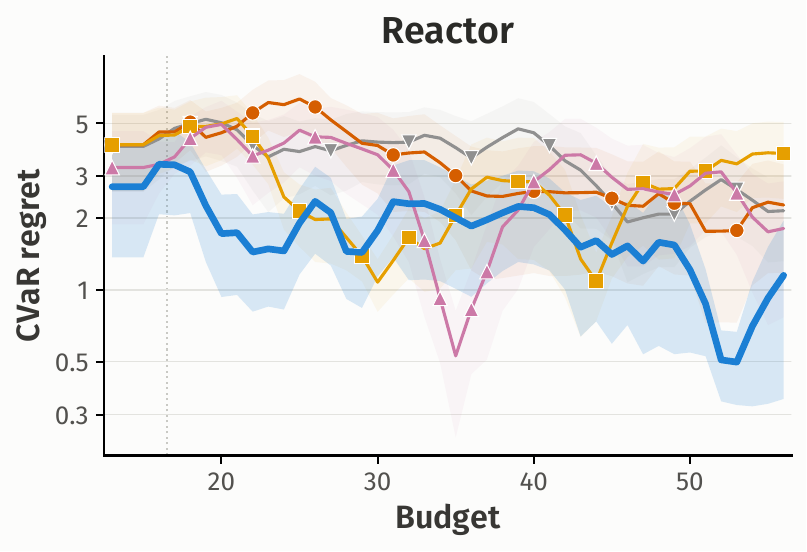}
\end{subfigure}
\hfill
\begin{subfigure}{0.32\linewidth}
    \centering
    \includegraphics[width=\linewidth]{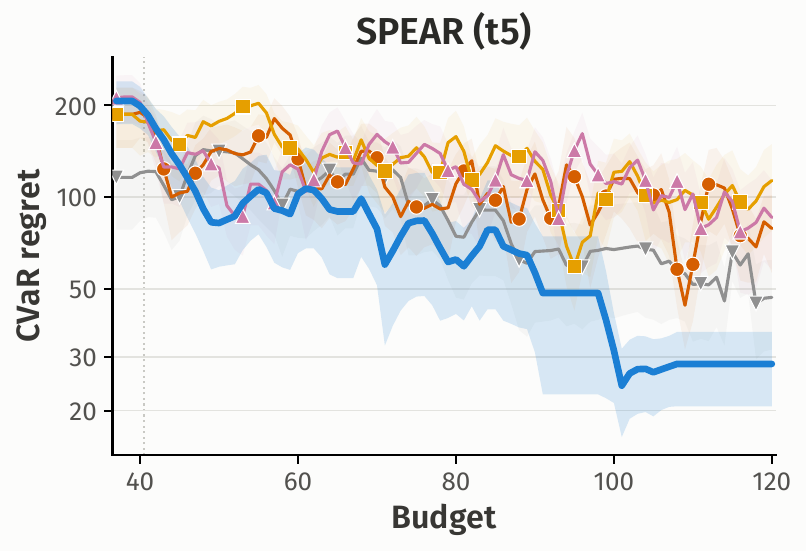}
\end{subfigure}
    \caption{CVaR regret comparing to from-scratch single-task models (no transfer). Mean $\pm$ 1 SEM over 10 seeds.}
    \label{fig:exp_mle}
\end{figure}

\paragraph{Exothermic batch reactor.}
We tune the operating conditions of a chemical reactor whose reaction releases heat~\citep{misicko-iecr26a}. The decision $x \in \mathcal{X}$ ($d_x = 3$) is an operating recipe (reaction rate, reactant concentration, cooling intensity), and an uncontrollable disturbance $W$ ($d_W = 3$) captures what varies from run to run: excess heat from impurities, degraded cooling, and initial-temperature uncertainty. The loss is the operating cost plus a large penalty whenever the peak temperature exceeds a safety threshold. Because the cheapest recipes are also those most likely to incur this rare but catastrophic failure, minimizing average cost can be unsafe, whereas $\mathrm{CVaR}_\alpha$ explicitly accounts for this tail risk.
We derive the sources by varying the reactor size while keeping the chemistry fixed, mirroring industrial scale-up from small bench reactors to an unseen production vessel. This induces systematic misspecification: as volume grows faster than surface area, cooling becomes less effective, so models fit on smaller reactors underestimate the target's failure tail.

This represents the severe end of the misspecification spectrum. \our attains the lowest final CVaR regret (Figure~\ref{fig:exp_realworld}, middle). BMA-UCB is competitive early on, but its likelihood weighting favors source models that underestimate tail risk, causing performance to deteriorate later. The remaining baselines commit to individual source models that fail to capture the target's failure region.

\paragraph{Algorithm configuration.}
We configure the \textsc{Spear} SAT solver ($d_x = 26$)~\citep{babic-sd07a} on the SWV software-verification benchmark from AClib~\citep{hutter-lion14a}, using its pre-computed runtime matrix over 1000 configurations and 604 instances. The decision $x$ is a solver configuration, the environment $W$ is a problem instance described by its SATzilla features, and the loss is the configured solver's runtime, right-censored at 300s. Runtimes are heavy-tailed: a configuration solves most instances quickly but can be orders of magnitude slower on a few hard ones, so the average hides exactly the failures $\mathrm{CVaR}_\alpha$ targets. We derive the sources by varying the instance population: runtime data from related instance distributions is available, while the target distribution is unseen. We cluster the instances into $K=8$ structural sub-families, hold one out as the target, and use each remaining family as a source task, requiring the solver to be configured for an unseen instance distribution.

As shown in Figure~\ref{fig:exp_realworld} (right), \our achieves the lowest average CVaR regret. The competing methods perform similarly and remain substantially worse, with final CVaR regrets above 50. This gap reflects the difficulty of selecting a single source model: likelihood-based selection can favor models that perform well on the majority of instances while failing to capture the rare hard instances that dominate the CVaR objective. By retaining multiple plausible surrogates and querying where they disagree, \our avoids committing prematurely to a source model that underestimates the failure tail.

\subsection{Value of transfer models (RQ4).}
To quantify the benefit of transferred candidate models, we compare \our against single-task surrogates fitted from scratch on the target by MLE, corresponding to the original setting of CV-UCB, CV-TS, and $\rho\text{-KG}^{\mathrm{apx}}$ together with random search. This comparison isolates whether access to related source tasks improves CVaR optimization over learning a surrogate solely from target observations. As can be seen in Figure~\ref{fig:exp_mle}, \our reaches substantially lower final CVaR regret than the best from-scratch model on the reactor and SAT tasks. On the robot pushing task, by contrast, CV-TS performs best. These results show that the benefit of transfer depends on how difficult it is to identify the tail-relevant structure of the target task: when the tail is easy to resolve from target observations alone, a single surrogate can suffice, whereas transfer becomes valuable when the tail behavior is difficult to infer from limited target data.

\section{Conclusion}
Optimizing CVaR requires accurately estimating the tail of the outcome distribution. When the GP surrogate is misspecified, tail estimates can be severely distorted, leading to poor risk-sensitive decisions. Existing CVaR Bayesian optimization algorithms typically assume that the correct surrogate model and its hyperparameters are known, an assumption that rarely holds in practice.

We addressed this limitation with \our, the first no-regret algorithm for CVaR Bayesian optimization with unknown kernel hyperparameters. Instead of committing to a single surrogate, \our maintains a finite set of plausible GP models, eliminates models that become inconsistent with the observations, queries where the surviving models disagree, and returns a decision that is robust to the remaining model uncertainty. We proved a cumulative regret bound matching the oracle rate of CV-UCB up to a factor linear in the candidate-set size, and demonstrated empirically that \our achieves significant improvements over competing methods while never being significantly outperformed.

\paragraph{Limitations and future work.}
Our theoretical guarantee assumes that a correctly specified surrogate is contained in the candidate set. The real-world experiments suggest that \our remains effective beyond this setting, but deriving regret guarantees that degrade with the misspecification of the best candidate, in the spirit of \citet{bogunovic-neurips21a}, remains an open direction. Our regret bound also incurs a factor $|\Theta|$ due to the two-sided elimination procedure, compared to the $\sqrt{|\Theta|}$ dependence obtained for single-criterion elimination; whether this gap is fundamental remains open. Finally, the proposed elimination framework is independent of the acquisition strategy, suggesting extensions to other risk-aware Bayesian optimization methods such as Thompson sampling and knowledge-gradient approaches.

\newpage
\bibliographystyle{plainnat}
\bibliography{strings,references,lib,proc}

\clearpage
\appendix

\begin{center}
    {\LARGE\bfseries Appendix}
\end{center}
\vspace{1em}

\renewcommand{\thesection}{\Alph{section}}
\renewcommand{\thesubsection}{\thesection.\arabic{subsection}}

\section{Proofs}
\label{app:proofs}

In this section we derive a regret bound for our proposed algorithm. The proof of Theorem~\ref{thm:regret} conditions on two probabilistic events that hold jointly with high probability: a GP confidence band for the true hyperparameter $\theta^\star$ (Lemma~\ref{lem:gp-confidence}), and a noise concentration bound on the side-stratified residual sums (Lemma~\ref{lem:noise-concentration}). We proceed in five steps: $\theta^\star$ is never eliminated (Step~1); the per-round regret is bounded by the band gap at the query (Step~2); the iterations at which eliminations fire contribute at most an additive constant (Step~3); the residual sums over the remaining iterations are controlled by the elimination thresholds themselves (Step~4); and the pieces are assembled (Step~5).

\paragraph{Lemma~\ref{lem:cvar-shift} (recalled).}
For integrable random variables $X,Y$ on a common probability space and
$\alpha\in(0,1)$,
\[
\mathrm{CVaR}_{\alpha}[X]-\mathrm{CVaR}_{\alpha}[Y]
\leq
\operatorname*{ess\,sup}(X-Y).
\]

\begin{proof}
Let $\delta=\operatorname*{ess\,sup}(X-Y)$. If $\delta=\infty$, the result is
trivial. Otherwise, $X\leq Y+\delta$ almost surely. By monotonicity of CVaR,
\[
\mathrm{CVaR}_{\alpha}[X]
\leq
\mathrm{CVaR}_{\alpha}[Y+\delta].
\]
Using translation invariance of CVaR,
\[
\mathrm{CVaR}_{\alpha}[Y+\delta]
=
\mathrm{CVaR}_{\alpha}[Y]+\delta,
\]
which proves the claim.
\end{proof}

\begin{lemma}[GP confidence band, \citealp{chowdhury-icml17a}]
\label{lem:gp-confidence}
Let $f \in \mathcal{H}_k$ with $\|f\|_{\mathcal{H}_k} \le B$ for a kernel $k \le 1$, and let $\varepsilon_t$ be conditionally $\sigma_\varepsilon$-sub-Gaussian. Let $\mu_{t-1}, \sigma_{t-1}$ denote the posterior mean and standard deviation of a GP with kernel $k$ and noise variance $\sigma_\varepsilon^2$ after $t-1$ queries, and set $\beta_t = B + \sigma_\varepsilon\sqrt{2(\gamma_{t-1} + 1 + \ln(1/\delta))}$, where $\gamma_{t-1}$ is the maximum information gain of the kernel after $t-1$ queries. Then with probability at least $1-\delta$,
\begin{equation*}
\bigl|\,f(x) - \mu_{t-1}(x)\,\bigr| \;\le\; {\beta_t}\,\sigma_{t-1}(x) \qquad \forall x \in \mathcal{X},\; t \ge 1.
\end{equation*}
\end{lemma}

Applying Lemma~\ref{lem:gp-confidence} to the $\theta^\star$-GP on the joint domain $\mathcal{X} \times \mathcal{W}$ with confidence level $\delta_A$ gives \textbf{Event A}: for all $(x,w) \in \mathcal{X} \times \mathcal{W}$ and all $t \ge 1$,
\begin{equation}
\label{eq:event-A}
\bigl|\,f(x,w) - \mu_{t-1}^{\theta^\star}(x,w)\,\bigr| \;\le\; {\beta_t^{\theta^\star}}\,\sigma_{t-1}^{\theta^\star}(x,w).
\end{equation}

For the second event, we adapt the noise concentration lemma of \citet[Lemma~5.1]{ziomek-aistats25a} to our side-stratified setting. Recall the stratified index sets $S^\theta_t = \{i \le t : \theta_i = \theta\}$ and $\hat S^\theta_t = \{i \le t : \vartheta_i = \theta\}$. The only modification relative to \citet{ziomek-aistats25a} is a union bound over both stratified sets per candidate, which contributes the factor $2|\Theta|$ inside the logarithm of $\xi_t$.

\paragraph{Lemma~\ref{lem:noise-concentration} (recalled).}
With $\xi_t = 2\sigma_\varepsilon^2 \log(2|\Theta|\pi^2 t^2 / 3\delta_B)$, $\mathbb{P}\bigl(\,\forall \theta \in \Theta,\;\forall t \le T,\; \forall S \in \{S^\theta_t, \hat S^\theta_t\}:\; \bigl|\,{\textstyle\sum_{i \in S}} \varepsilon_i\,\bigr| \le \sqrt{\xi_t\,|S|} \,\bigr) \;\ge\; 1 - \delta_B.$

\begin{proof}
Fix $\theta \in \Theta$ and a side $S \in \{S^\theta_t, \hat S^\theta_t\}$. Let $v_1, \ldots, v_T$ be an i.i.d.\ sequence of $\sigma_\varepsilon$-sub-Gaussian variables, sampled before the experiment as hidden state. Whenever $S$ is expanded at some iteration $t$, set the noise of that iteration to $\varepsilon_t = v_i$ with $i = |S|$. The constructed noise has the same distribution as the original i.i.d.\ noise.

The standard sub-Gaussian Hoeffding bound \citep{lattimore-cambridge20} gives, for any $n \le T$,
\begin{equation*}
\mathbb{P}\Bigl(\Bigl|{\textstyle\sum_{i=1}^n} v_i\Bigr| \ge \sqrt{An}\,\Bigr) \;\le\; 2 \exp\!\Bigl(-\tfrac{A}{2\sigma_\varepsilon^2}\Bigr).
\end{equation*}
Define $E_n := \bigl\{\bigl|\sum_{i=1}^n v_i\bigr| > \sqrt{\xi_n n}\bigr\}$. With $\xi_n = 2\sigma_\varepsilon^2 \log(2|\Theta|\pi^2 n^2 / 3\delta_B)$, the Hoeffding bound gives
\begin{equation*}
\mathbb{P}(E_n) \;\le\; 2\exp\!\Bigl(-\tfrac{\xi_n}{2\sigma_\varepsilon^2}\Bigr) \;=\; 2\cdot\frac{3\delta_B}{2|\Theta|\pi^2 n^2} \;=\; \frac{3\delta_B}{|\Theta|\pi^2 n^2}.
\end{equation*}
Since $|S| \le T$,
\begin{align*}
\mathbb{P}\Bigl(\exists t:\, \Bigl|{\textstyle\sum_{i \in S}}\varepsilon_i\Bigr| > \sqrt{\xi_t |S|}\Bigr) \;=\; \mathbb{P}\Bigl({\textstyle\bigcup_{n=1}^{|S|}} E_n\Bigr) \\\;\le\; \mathbb{P}\Bigl({\textstyle\bigcup_{n=1}^T} E_n\Bigr) 
\;\le\; {\textstyle\sum_{n=1}^T} \mathbb{P}(E_n)\\ \;\le\; \frac{3\delta_B}{|\Theta|\pi^2}\sum_{n=1}^\infty \frac{1}{n^2} \;=\; \frac{3\delta_B}{|\Theta|\pi^2}\cdot\frac{\pi^2}{6} \;=\; \frac{\delta_B}{2|\Theta|}.
\end{align*}
A union bound over the $2|\Theta|$ choices of $(\theta, S)$ gives total failure probability at most $\delta_B$.
\end{proof}
We call the event in Lemma~\ref{lem:noise-concentration} \textbf{Event B.} By a union bound, Events A and B hold jointly with probability at
least $1 - \delta$ where we set $\delta_A = \delta_B = \delta/2$.
The remainder of the proof is deterministic conditional on this
good event. We now restate the theorem and give the full proof.

\begin{theorem}[Cumulative regret of \our, restated]
\label{thm:RRA-UCB-app}
Assume $f \in \mathcal{H}_{k_{\theta^\star}}$ with
$\|f\|_{\mathcal{H}_{k_{\theta^\star}}} \le B$ for some
$\theta^\star \in \Theta$ and that
$\varepsilon_t \sim \mathcal{N}(0, \sigma_\varepsilon^2)$ i.i.d.,
with $k_\theta \le 1$ for all $\theta \in \Theta$. Set $\beta_t^\theta = B + \sigma_\varepsilon\sqrt{2(\gamma_{t-1}^\theta + 1 + \ln(2/\delta))}$ and $\xi_t = 2\sigma_\varepsilon^2 \log(4|\Theta|\pi^2 t^2 / 3\delta)$. Then with probability at least $1 - \delta$ the cumulative regret of
Algorithm~\ref{alg:ouralg} after $T$ iterations satisfies
\begin{equation}
\label{eq:he-cv-regret-app}
R_T \;\le\; \mathcal{O}\!\Bigl(\,|\Theta|\, B + |\Theta|\sqrt{\xi_T\, T} + \beta_T\,|\Theta|\sqrt{T\, \gamma_T}\,\Bigr),
\end{equation}
where $\beta_T := \max_{\theta \in \Theta} \beta_T^\theta$ and
$\gamma_T := \max_{\theta \in \Theta} \gamma_T^\theta$.
\end{theorem}

The argument has five steps. Steps~1--2 reduce the per-round regret to a quantity the algorithm observes: the correct candidate $\theta*$ is never eliminated (Step~1), so its confidence bands remain valid and bound one round's regret by the gap between the two selected candidates at the query (Step~2). Steps~3--5 bound the residual part over the horizon: the few rounds where an elimination fires are charged a constant (Step~3), the residuals on all other rounds are controlled by the elimination thresholds themselves (Step~4), and the pieces are assembled into the rate (Step~5)

\paragraph{Step 1: Preservation of $\theta^\star$.} 

We first show the correct candidate is never discarded, so it remains available as a valid reference. Specifically, we show that neither the upper-side nor the lower-side elimination
criterion can fire on $\theta^\star$. Consider any iteration $t$ at
which $\theta_t = \theta^\star$ and the criterion
\eqref{eq:he-cv-elim-u} is checked. For each $i \in S^{\theta^\star}_t$
the residual $\eta_i$ decomposes as
\begin{equation*}
\eta_i \;=\; y_i - \mu_{i-1}^{\theta^\star}(x_i, w_i)
       \;=\; \varepsilon_i + \bigl(f(x_i, w_i) - \mu_{i-1}^{\theta^\star}(x_i, w_i)\bigr).
\end{equation*}
Applying the triangle inequality together with Events A and B,
\begin{align*}
\Bigl|\,{\textstyle\sum_{i \in S^{\theta^\star}_t}} \eta_i\,\Bigr|
 \;\\\le\; \Bigl|\,{\textstyle\sum_{i \in S^{\theta^\star}_t}} \varepsilon_i\,\Bigr|
       + {\textstyle\sum_{i \in S^{\theta^\star}_t}}\!\bigl|\,f(x_i, w_i) - \mu_{i-1}^{\theta^\star}(x_i, w_i)\,\bigr| \\
 \;\le\; \sqrt{\xi_t\,|S^{\theta^\star}_t|}
       + {\textstyle\sum_{i \in S^{\theta^\star}_t}}\!{\beta_i^{\theta^\star}}\,\sigma_{i-1}^{\theta^\star}(x_i, w_i),
\end{align*}
which is exactly the right-hand side of \eqref{eq:he-cv-elim-u}.
Hence the criterion does not fire and $\theta^\star$ is not eliminated from the upper side. For the lower side the argument is identical with $\eta_i$ replaced by $\hat\eta_i = y_i - \mu_{i-1}^{\theta^\star}(x_i, w_i)$ and $S^{\theta^\star}_t$ replaced by $\hat S^{\theta^\star}_t$. Therefore $\theta^\star$ is never eliminated, and $\theta^\star \in \Theta_t$ for all $t \ge 0$.

\paragraph{Step 2: Per-round regret.}
Since $\theta^\star \in \Theta_{t-1}$ for all $t$ (Step 1), Event~A combined with the integral representation of CVaR \citep{nguyenquoc-neurips21a} gives, for every $x \in \mathcal{X}$,
\begin{equation*}
c_{l_{t-1}}^{\theta^\star}(x; \alpha) \;\le\; \mathrm{CVaR}_\alpha[f(x, W)] \;\le\; c_{u_{t-1}}^{\theta^\star}(x; \alpha).
\end{equation*}
We bound the two terms of $r_t = \mathrm{CVaR}_\alpha[f(x_t, W)] - \mathrm{CVaR}_\alpha[f(x^\star, W)]$ separately.

\emph{Upper bound at $x_t$.} By the pessimistic rule \eqref{eq:he-cv-vt}, $\vartheta_t$ maximises $c_{u_{t-1}}^{\theta}(x_t; \alpha)$ over $\Theta_{t-1}$; since $\theta^\star \in \Theta_{t-1}$ is a feasible choice,
\begin{equation*}
\mathrm{CVaR}_\alpha[f(x_t, W)] \;\le\; c_{u_{t-1}}^{\theta^\star}(x_t; \alpha) \;\le\; c_{u_{t-1}}^{\vartheta_t}(x_t; \alpha).
\end{equation*}

\emph{Lower bound at $x^\star$.} By the doubly-optimistic rule \eqref{eq:he-cv-xt}, $(x_t, \theta_t)$ minimises $c_{l_{t-1}}^{\theta}(x; \alpha)$ over $\mathcal{X} \times \Theta_{t-1}$; since $(x^\star, \theta^\star)$ is one element of this set,
\begin{equation*}
c_{l_{t-1}}^{\theta_t}(x_t; \alpha) \;\le\; c_{l_{t-1}}^{\theta^\star}(x^\star; \alpha) \;\le\; \mathrm{CVaR}_\alpha[f(x^\star, W)].
\end{equation*}

Combining the two,
\begin{equation}
\label{eq:cvar-gap}
r_t \;\le\; c_{u_{t-1}}^{\vartheta_t}(x_t; \alpha) - c_{l_{t-1}}^{\theta_t}(x_t; \alpha).
\end{equation}
Both steps combine only the defining property of the maximiser or minimiser with the validity of the $\theta^\star$-band; no assumption is placed on the bands of any other candidate.

The right-hand side of \eqref{eq:cvar-gap} is the difference of the CVaRs of two random variables on the common probability space $(\mathcal{W}, \mathbb{P}_W)$. Applying Lemma~\ref{lem:cvar-shift} with $X = u_{t-1}^{\vartheta_t}(x_t, W)$ and $Y = l_{t-1}^{\theta_t}(x_t, W)$, and recalling $G_t(w) = u_{t-1}^{\vartheta_t}(x_t, w) - l_{t-1}^{\theta_t}(x_t, w)$,
\begin{align}
c_{u_{t-1}}^{\vartheta_t}(x_t; \alpha) - c_{l_{t-1}}^{\theta_t}(x_t; \alpha) \;\le\; \operatorname*{ess\,sup}_{W} G_t(W) \\ \;\le\; \max_{w \in \mathcal{W}} G_t(w) \;=\; G_t(w_t),
\end{align}
where the last two steps use that the pointwise maximum dominates the essential supremum and that $w_t$ attains it by the lacing-value rule \eqref{eq:he-cv-lv}. Hence
\begin{equation}
\label{eq:band-gap}
r_t \;\le\; u_{t-1}^{\vartheta_t}(x_t, w_t) - l_{t-1}^{\theta_t}(x_t, w_t).
\end{equation}
It remains to expand the right-hand side. Using $u_{t-1}^{\theta} = \mu_{t-1}^{\theta} + \beta_t^{\theta} \sigma_{t-1}^{\theta}$ and $l_{t-1}^{\theta} = \mu_{t-1}^{\theta} - \beta_t^{\theta} \sigma_{t-1}^{\theta}$, together with $\mu_{t-1}^{\theta_t}(x_t, w_t) = y_t - \eta_t$ and $\mu_{t-1}^{\vartheta_t}(x_t, w_t) = y_t - \hat\eta_t$,
\begin{align}
\label{eq:per-round-bound}
r_t \;\le\;\; & u_{t-1}^{\vartheta_t}(x_t, w_t) - l_{t-1}^{\theta_t}(x_t, w_t) \nonumber \\
 \;=\;\; & \bigl(\mu_{t-1}^{\vartheta_t}(x_t, w_t) - \mu_{t-1}^{\theta_t}(x_t, w_t)\bigr) \nonumber \\
   & + \beta_t^{\vartheta_t}\sigma_{t-1}^{\vartheta_t}(x_t, w_t) + \beta_t^{\theta_t}\sigma_{t-1}^{\theta_t}(x_t, w_t) \nonumber \\
 \;=\;\; & (\eta_t - \hat\eta_t) + \beta_t^{\theta_t}\sigma_{t-1}^{\theta_t}(x_t, w_t) + \beta_t^{\vartheta_t}\sigma_{t-1}^{\vartheta_t}(x_t, w_t).
\end{align}
\paragraph{Step 3: Splitting into critical and non-critical iterations.} Rounds where an elimination fires can incur large regret, however, because we have finitely many models, their total cost is not growing with $T$. Define the set of \emph{critical} iterations as those where at
least one of the elimination criteria fires:
\begin{equation*}
\mathcal{C} \;:=\; \bigl\{\, 1 \le t \le T :
     \text{\eqref{eq:he-cv-elim-u} or \eqref{eq:he-cv-elim-l}
     triggers at $t$} \,\bigr\}.
\end{equation*}
Each $t \in \mathcal{C}$ removes at least one element from
$\Theta_t$. Since $\theta^\star$ is never eliminated (Step 1), at
most $|\Theta| - 1$ elements can be removed across all iterations,
so $|\mathcal{C}| \le |\Theta| - 1 \le |\Theta|$.

On critical iterations the algorithm may incur large per-round
regret: a wrong surrogate selected optimistically can lead the
algorithm to query a poor point. We absorb this into a trivial
bound. Since $k_{\theta^\star} \le 1$ and
$\|f\|_{\mathcal{H}_{k_{\theta^\star}}} \le B$, we have
$|f(x, w)| \le B$ for all $(x, w)$, hence
$|\mathrm{CVaR}_\alpha[f(x, W)]| \le B$ and $r_t \le 2B$ for any
$t$. Crucially, although the regret on each critical iteration can
be large, the \emph{number} of critical iterations is bounded by
$|\Theta|$ independent of the horizon $T$. The cost of elimination
is therefore an additive constant in $T$, not a term that grows
with the horizon.

\paragraph{Step 4: Bounding the residual sums on non-critical iterations.}
Fix a candidate $\theta$ and let $S^\theta_T = \{t_1 < t_2 < \cdots\}$ be the rounds at which $\theta$ was the optimistic selection ($\theta_t = \theta$), in time order; we refer to this as $\theta$'s \emph{stream}. Write $P_j := \sum_{k \le j} \eta_{t_k}$ for the running residual total after $\theta$'s first $j$ selections, and call a subset of the stream an \emph{initial run} if it is a prefix $\{t_1, \dots, t_j\}$, that is, if it contains no gaps. 

In our setting each round feeds two streams, an optimistic and a pessimistic one, and a round is critical whenever \emph{either} fires. Consequently, when we sum $\sum_{t \notin \mathcal{C}} \eta_t$ over the optimistic stream, a round $i$ may be excluded because the \emph{lower} criterion fired there, potentially on a different candidate, even though $\theta$ itself survived and was selected again afterwards. Such a round lies in the interior of $\theta$'s stream, so the surviving rounds no longer form an initial ru. The rest of this step controls exactly these interior omissions. Let
\begin{equation*}
\Phi^\theta_T \;:=\; \sqrt{\xi_T\,|S^\theta_T|} + \sum_{i \in S^\theta_T} {\beta_i^\theta}\,\sigma_{i-1}^\theta(x_i, w_i)
\end{equation*}
denote the elimination threshold from \eqref{eq:he-cv-elim-u}, which is non-decreasing in $T$.

\emph{Running totals of a surviving candidate stay under the threshold.} The criterion \eqref{eq:he-cv-elim-u} compares exactly the running total $|P_j|$ against $\Phi^\theta_{t_j}$ at each selection $t_j$. If $\theta$ is still alive after round $t_j$, then the criterion did not fire there, so $|P_j| \le \Phi^\theta_{t_j} \le \Phi^\theta_T$ by monotonicity. Since $\theta$ can be eliminated by this criterion at most once, and only at its final selection $\max S^\theta_T$, the inequality $|P_j| \le \Phi^\theta_T$ holds for every $j$ except possibly the last.

\emph{Two kinds of excluded rounds.} We want $\sum_{t \notin \mathcal{C}} \eta_t$, the residual sum over rounds where \emph{no} elimination fired, so two kinds of $\theta$-rounds drop out. First, $\theta$ may have been eliminated itself, by \eqref{eq:he-cv-elim-u}; this happens at its last selection, so dropping that round leaves the initial run $t_1, \dots, t_{j-1}$, whose total $P_{j-1}$ is already bounded by $\Phi^\theta_T$. Second, a round $t_j \in S^\theta_T$ may drop out because the \emph{other} criterion \eqref{eq:he-cv-elim-l} eliminated a \emph{different} candidate $\vartheta_{t_j} \ne \theta$ there. Such a round can sit anywhere in $\theta$'s sequence, so removing it leaves a gap rather than an initial run, and the running-total bound no longer applies to the remaining sum directly.

\emph{Recovering an omitted residual.} For such an excluded round $t_j$ --- never the last selection of $\theta$, hence a round at which \eqref{eq:he-cv-elim-u} did not fire --- both $P_j$ and $P_{j-1}$ are under the threshold, and the lone residual is their difference,
\begin{equation}
\label{eq:bystander-eta}
|\eta_{t_j}| \;=\; |P_j - P_{j-1}| \;\le\; |P_j| + |P_{j-1}| \;\le\; 2\,\Phi^\theta_T .
\end{equation}

\emph{Assembling the sum.} Let $F_u$ and $F_l$ be the rounds at which \eqref{eq:he-cv-elim-u} and \eqref{eq:he-cv-elim-l} fire, so $\mathcal{C} = F_u \cup F_l$ and $\{t \notin \mathcal{C}\} = \{t \notin F_u\} \setminus (F_l \setminus F_u)$. Summing over candidates,
\begin{equation*}
\sum_{t \notin \mathcal{C}} \eta_t \;=\; \sum_{t \notin F_u} \eta_t \;-\; \sum_{t \in F_l \setminus F_u} \eta_t .
\end{equation*}
The first sum is $\sum_\theta \sum_{i \in S^\theta_T \setminus F_u} \eta_i$, an initial run for each $\theta$, hence at most $\sum_\theta \Phi^\theta_T$ by the running-total bound. The second ranges over the rounds excluded by the opposite criterion; there are at most $|F_l \setminus F_u| \le |\mathcal{C}| \le |\Theta|$ of them, each bounded by \eqref{eq:bystander-eta}. Writing $b^u_\theta := |(F_l \setminus F_u) \cap S^\theta_T|$ (so $\sum_\theta b^u_\theta \le |\Theta|$),
\begin{equation}
\label{eq:noncrit-eta}
\Bigl|\,{\textstyle\sum_{t \notin \mathcal{C}}} \eta_t\,\Bigr| \;\le\; {\textstyle\sum_{\theta}} \Phi^\theta_T \;+\; 2 {\textstyle\sum_{\theta}} b^u_\theta\,\Phi^\theta_T .
\end{equation}
The same bound holds for $|\sum_{t \notin \mathcal{C}}\hat\eta_t|$ with $(\hat\eta, \hat S^\theta_T, \hat\Phi^\theta_T)$ and counts $\hat b_\theta := |(F_u \setminus F_l) \cap \hat S^\theta_T|$, $\sum_\theta \hat b_\theta \le |\Theta|$. The extra term beyond the single-criterion analysis of \citet{ziomek-aistats25a} is the second sum in \eqref{eq:noncrit-eta}: with two criteria, an elimination on one side can remove a mid-sequence round from the other side's residual stream, and \eqref{eq:bystander-eta} charges each of the at most $|\Theta|$ such removals two thresholds.

\paragraph{Step 5: Cumulative regret.} We now combine the three contributions; the few costly elimination rounds, the residual sums, and the usual GP-uncertainty widths into the final rate. With $r_t \le 2B$ on $\mathcal{C}$ (Step 3) and the per-round bound
\eqref{eq:per-round-bound} on $t \notin \mathcal{C}$,
\begin{align*}
R_T \;\le\; \underbrace{2|\mathcal{C}|\,B}_{\text{(A)}}
   \;+\; \underbrace{\Bigl|{\textstyle\sum_{t \notin \mathcal{C}}}\hat\eta_t\Bigr|
                   + \Bigl|{\textstyle\sum_{t \notin \mathcal{C}}}\eta_t\Bigr|}_{\text{(B)}}
   \;\\+\; \underbrace{{\textstyle\sum_{t \notin \mathcal{C}}}\bigl(
        {\beta_t^{\theta_t}}\sigma_{t-1}^{\theta_t}
       + {\beta_t^{\vartheta_t}}\sigma_{t-1}^{\vartheta_t}\bigr)(x_t,w_t)}_{\text{(C)}} .
\end{align*}

\emph{Two recurring bounds.} Cauchy--Schwarz with
$\sum_\theta|S^\theta_T| = T$ gives
$\sum_\theta\sqrt{\xi_T|S^\theta_T|} \le \sqrt{|\Theta|\,\xi_T\,T}$;
and $\beta_t^\theta\le\beta_T$ with the per-surrogate
maximum-information-gain bound of \citet[Lemma~5.4]{srinivas-icml10a},
$\sum_{i\in S^\theta_T}(\sigma_{i-1}^\theta)^2 \le C\gamma_T$, gives both
$\sum_\theta\sum_{i\in S^\theta_T}{\beta_i^\theta}\sigma_{i-1}^\theta
   \le \beta_T\,\sqrt{C\,|\Theta|\,T\,\gamma_T}$ and, per surrogate,
$\sum_{i\in S^\theta_T}{\beta_i^\theta}\sigma_{i-1}^\theta
   \le \beta_T\,\sqrt{C\,\gamma_T\,T}$. Hence
\begin{align*}
{\textstyle\sum_\theta}\Phi^\theta_T \le
   \sqrt{|\Theta|\,\xi_T\,T} + \beta_T\,\sqrt{C\,|\Theta|\,T\,\gamma_T},
\\
\Phi^\theta_T \le \sqrt{\xi_T\,T}+\beta_T\,\sqrt{C\,\gamma_T\,T} =: \Phi^\star_T
\;\;\forall\theta,
\end{align*}
and the same two bounds hold for $\hat\Phi^\theta_T$ on the lower side.

\emph{Term (A).} By Step 3, $|\mathcal{C}| \le |\Theta|$, so
$\text{(A)} \le 2|\Theta|\,B$.

\emph{Term (B).} By \eqref{eq:noncrit-eta}, its lower-side analogue, and
$\sum_\theta b^u_\theta,\ \sum_\theta \hat b_\theta \le |\Theta|$,
\begin{align*}
\text{(B)} \;\le\; {\textstyle\sum_\theta}\Phi^\theta_T
   + {\textstyle\sum_\theta}\hat\Phi^\theta_T + 4|\Theta|\,\Phi^\star_T
 \;\\\le\; 2\sqrt{|\Theta|\,\xi_T\,T}
       + 2\beta_T\,\sqrt{C\,|\Theta|\,T\,\gamma_T}
       + 4|\Theta|\,\Phi^\star_T .
\end{align*}

\emph{Term (C).} As in the original argument, $\beta_t^{\theta_t}\le\beta_T$
and the per-surrogate information-gain bound give
$\text{(C)} \le 2\beta_T\,\sqrt{C\,|\Theta|\,T\,\gamma_T}$.

\emph{Combining.} Substituting
$\Phi^\star_T \le \sqrt{\xi_T T}+\beta_T\,\sqrt{C\,\gamma_T\,T}$, the
bystander contribution $4|\Theta|\,\Phi^\star_T$ dominates every
$\sqrt{|\Theta|\,\cdot}$ term, so
\begin{align*}
R_T \;\le\;\;  2|\Theta|\,B
   + 4|\Theta|\bigl(\sqrt{\xi_T\,T} + \beta_T\,\sqrt{C\,\gamma_T\,T}\bigr)
   \\+ \mathcal{O}\!\bigl(\sqrt{|\Theta|\,\xi_T\,T} + \beta_T\,\sqrt{|\Theta|\,T\,\gamma_T}\bigr) \\
 \;=\;\;  \mathcal{O}\!\Bigl(\,|\Theta|\,B
   + |\Theta|\sqrt{\xi_T\,T}
   + \beta_T\,|\Theta|\sqrt{T\,\gamma_T}\,\Bigr),
\end{align*}
which is the claimed bound \eqref{eq:he-cv-regret-app}. 
For a fixed finite candidate set, the bound implies sublinear
regret whenever $\gamma_T = o(\sqrt{T})$, since
$\beta_T = \mathcal{O}(\sqrt{\gamma_T + 1})$ and
$\xi_T = \mathcal{O}(\log T)$ for the standard kernels of
\citet{srinivas-icml10a}. \qed

 \begin{remark}[On the $|\Theta|$ versus $\sqrt{|\Theta|}$ dependence]
The bound \eqref{eq:he-cv-regret-app} carries $|\Theta|$ as a prefactor on the two horizon-dependent terms, whereas the single-criterion bound of \citet{ziomek-aistats25a} carries only $\sqrt{|\Theta|}$. The discrepancy is due entirely to the correction term in \eqref{eq:noncrit-eta}, and its source is the initial-run property of Step~4. With a single elimination criterion, a candidate's surviving rounds always form an initial run---it is excluded from its own residual sum only when it is itself eliminated, at its last selection---so the sum is controlled in one step by the running-total bound. With two criteria the property fails: a round may be excluded because the \emph{opposite} criterion eliminated a \emph{different} candidate, leaving a gap in the interior of the stream. Each such round must then be bounded individually rather than pooled, and \eqref{eq:bystander-eta} charges it the full threshold $\Phi^\star_T = \mathcal{O}(\sqrt{\xi_T\,T} + \beta_T\sqrt{\gamma_T\,T})$. There are up to $|\Theta|$ of them, so their aggregate $\mathcal{O}(|\Theta|\,\Phi^\star_T)$ multiplies the threshold by the count \emph{linearly}, which dominates the $\sqrt{|\Theta|}$ terms and yields the $|\Theta|$ prefactor.
\end{remark}

\section{Experimental Details}
\label{app:exp}

This appendix gives the full experimental protocol: the representation of the
environmental distribution~$P_W$, the computation of the posterior CVaR bands,
the four benchmark problems and how their related source tasks are generated,
the per-iteration computational cost, the final hyper-parameters, and the
computing environment.

% ---------------------------------------------------------------------
\subsection{Environmental distribution $P_W$}
\label{app:conv}

We represent the context law $P_W$ by a fixed set of $M$ low-discrepancy (Sobol)
atoms $\{w_1,\dots,w_M\}$, drawn once at problem construction with a fixed seed
and reused across all designs $x$ and all iterations. Unless stated otherwise
$P_W$ is uniform on the context box; for problems whose disturbance concentrates
around a nominal operating point we instead place the atoms under a
truncated-normal law. This shared atom set defines the empirical tail average
over which every CVaR quantity below is computed.

% ---------------------------------------------------------------------
\subsection{Computing the posterior CVaR bands}
\label{app:bands}

Each candidate $\theta\in\Theta$ is a Gaussian-process surrogate with
\emph{frozen} hyper-parameters (kernel, length-scales, output-scale, noise);
conditioning on the $t$ observations yields, at any design $x$, a closed-form
per-atom posterior mean $\mu^\theta(x,w_j)$ and standard deviation
$\sigma^\theta(x,w_j)$ for each of the $M$ context atoms $w_j$. We form the two
bands \emph{analytically} from these per-atom moments: each band is the
empirical CVaR of the shifted posterior mean over the atoms, i.e.\ the mean of
its largest $k=\lceil(1-\alpha)M\rceil$ entries,
\begin{align}
  c_u^\theta(x;\alpha) = \widehat{\mathrm{CVaR}}_\alpha\big[\,\mu^\theta(x,W)+\beta\,\sigma^\theta(x,W)\,\big]
   \\ = \tfrac{1}{k}\!\!\sum_{j\in\mathrm{top}\text{-}k}\!\!\big(\mu^\theta_j+\beta\sigma^\theta_j\big),\\
  c_l^\theta(x;\alpha) = \widehat{\mathrm{CVaR}}_\alpha\big[\,\mu^\theta(x,W)-\beta\,\sigma^\theta(x,W)\,\big],
  \\k =\lceil(1-\alpha)M\rceil .
\end{align}
Because the atoms are fixed and shared across designs and iterations, both bands
are deterministic functionals of the GP posteriors: no posterior sampling is
involved, so acquisition and elimination compare candidates without Monte-Carlo
noise. The same atoms enter every candidate's band, so band differences across
designs or across candidates carry no atom-set variability.

% ---------------------------------------------------------------------
\subsection{Benchmark suite}
\label{app:bench}

Table~\ref{tab:bench} summarizes the four benchmarks; all use the same transfer
method set: \our against the single-selection
baselines (CV-UCB, CV-TS, $\rho\text{-KG}^{\mathrm{apx}}$), BMA-UCB and the band hyper-parameters
$\beta{=}3.0$, $\xi{=}0.5$.

\begin{table*}[t]
\centering
\small
\setlength{\tabcolsep}{4pt}
\begin{tabular}{lcccccccc}
\toprule
Benchmark & $d_x$ & $d_W$ & $|\Theta|$ & $n_{\mathrm{init}}$ & budget & seeds & $\alpha$ \\
\midrule
Spike (synthetic)          & 1  & 1  & 9  & 8  & 30  & 10 & 0.90   \\
Robot push (goal-transfer) & 3  & 2  & 8   & 10 & 60  & 10 & 0.70  \\
SPEAR (cluster-transfer)   & 26 & 10 & 7     & 40 & 80  & 10 & 0.70   \\
Reactor (scale-transfer)   & 3  & 3  & 9     & 16 & 40  & 10 & 0.98 \\
\bottomrule
\end{tabular}
\caption{Benchmark configurations. Budget is BO iterations after the initial
design; total evaluations $=n_{\mathrm{init}}+\text{budget}$. $\alpha$ is the
confidence-level CVaR parameter (worst $1-\alpha$ fraction).}
\label{tab:bench}
\end{table*}

\paragraph{Spike (mechanism illustration).}
A $2$-dimensional analytic testbed with a low-measure tail spike at
$w\!\approx\!0.85$ over a mean-optimal trap at $x\!\approx\!0.15$ and a
CVaR-safe design at $x\!\approx\!0.72$. This is a controlled illustration of the mechanism, \emph{not} a
real-world claim; the truth is analytic so the CVaR oracle is exact up to the
dense evaluation grid.

\paragraph{Robot push (goal-transfer).}
The single-robot Box2D pushing task of \citet{nguyenquoc-neurips21a}: the design
$x=(r_x,r_y,t_p)$ ($d_x{=}3$) sets the robot's start location and push duration,
and the context $W$ ($d_W{=}2$) is an uncontrollable perturbation of the robot's
start location, drawn from a truncated-normal actuation law centered at the
nominal start. The loss is the residual distance between the pushed object and a
fixed goal; small perturbations barely deflect the object while large ones cause
it to miss badly, which is the tail that $\mathrm{CVaR}_\alpha$ targets. The
transfer axis is the goal \emph{location}: we obtain the $K{=}8$ sources by
displacing the target goal to eight different nearby positions, and fit one
frozen GP expert per source goal from an independent risk-averse BO campaign on
that goal, snapshotted at budget $120$ ($=8$ experts). Because the push dynamics
are identical across all goals, each source captures the correct dynamics but
places the CVaR-optimal design at a slightly wrong location; this is the benign
end of the misspecification spectrum, where every source models the target
reasonably and committing to the wrong one costs little. 

\paragraph{Reactor (scale-up transfer).}
A cooled batch-reactor thermal-runaway model \citep{misicko-iecr26a}. The
design $x=(\mathrm{Da},\chi,c_0)$ ($d_x{=}3$) is an operating recipe (throughput
via the Damk\"ohler number, cooling multiplier, initial reactant concentration),
and the context $W=(\Delta B,\varphi,\theta_0)$ ($d_W{=}3$) is an uncontrollable
disturbance collecting a feed-impurity exotherm, a cooling-fault multiplier, and
control-noise on the initial temperature. The loss is the operating cost plus a
large penalty whenever the peak temperature exceeds the safety threshold; because
the cheapest recipes are also the most likely to trigger this rare but
catastrophic runaway, minimizing average cost is unsafe and
$\mathrm{CVaR}_\alpha$ is the appropriate objective. The transfer axis is reactor
\emph{scale}, encoded through the Stanton number, mirroring industrial scale-up
from small bench reactors to an unseen production vessel: as volume grows faster
than surface area, cooling becomes less effective, so experts fit on smaller
reactors systematically underestimate the target's failure tail. Each of the $9$
leave-one-out source scales contributes one frozen expert, fit from a risk-averse
BO campaign on that scale and snapshotted at budget $60$ ($=9$ experts). This is
the severe end of the spectrum. 

\paragraph{SPEAR (algorithm-configuration transfer).}
Real AClib data \citep{hutter-lion14a}: $604$ SWV software-verification SAT
instances $\times\,1000$ pre-evaluated \textsc{Spear} configurations
\citep{babic-sd07a}. The design $x$ is a $26$-parameter solver configuration
($d_x{=}26$); the context $W$ is a problem instance, represented by its SATzilla
features reduced to a $d_W{=}10$-dimensional PCA of the $138$ raw features (the
PCA keeps the acquisition and elimination search low-dimensional while retaining
the instance structure that drives runtime); and the loss is the configured
solver's runtime, right-censored at $300$ s. Runtimes are
heavy-tailed a configuration solves most instances quickly but can be orders of
magnitude slower on a few hard ones, so the average hides exactly the failures
$\mathrm{CVaR}_\alpha$ targets. The transfer axis is the instance population: we
$k$-means cluster the instances into $K{=}8$ structural sub-families, hold one out
as the target distribution over $W$, and let each of the remaining $7$ source
clusters contribute one frozen expert, fit from a risk-averse BO campaign on that
cluster and snapshotted at budget $300$ ($=7$ experts). The learner must thus
configure the solver for an unseen instance distribution. Ground-truth CVaR is
exact, computed over the held-out cluster's real instances.

% ---------------------------------------------------------------------
\subsection{Computational cost}
\label{app:cost}

\paragraph{Per-iteration work.}
Because pool members have \emph{frozen} hyper-parameters, there is \emph{no
per-iteration MLE refit}: each surviving member is a GP conditioned on the $t$
observations, costing one $O(t^3)$ Cholesky plus $O(M t^2)$ for the $M$ per-atom
posterior evaluations. One BO iteration therefore costs
$O\!\big(|\Theta_{\mathrm{surv}}|\,(t^3+M t^2)\big)$ plus the acquisition
optimization ($|\Theta_{\mathrm{surv}}|$ per-member $c_l$ minimizations). Cost
grows \emph{linearly} in $|\Theta|$, with $d_W$ through $M$ (the atom count), and
with $d_x$ only through kernel evaluations. Implementation economies:
frozen experts (no hyper-parameter optimization), analytic per-atom bands (no
posterior sampling), a shared fixed atom set reused across designs and
iterations, and elimination that \emph{monotonically shrinks}
$|\Theta_{\mathrm{surv}}|$ over a run.

% ---------------------------------------------------------------------
\subsection{Final hyper-parameters}
\label{app:hparams}

All values are fixed across seeds; none are tuned per seed or per run.
\begin{itemize}
\item \textbf{Band:} $\beta=3.0$ (raw $\sigma$-coefficient, i.e.\ band
  $\mu\pm\beta\sigma$).
\item \textbf{CVaR level $\alpha$:} per Table~\ref{tab:bench}
  ($\alpha\in\{0.70,0.90,0.98\}$) chosen by the problem's tail-mass rationale.
\item \textbf{Budget:} $n_{\mathrm{init}},\text{budget}$ per Table~\ref{tab:bench}.
\item \textbf{Experts:} kernels $\{$Mat\'ern-$3/2$, Mat\'ern-$5/2$, RBF$\}$, fitted on offline tasks.
\item \textbf{Recommender:} pessimistic ($\Gamma$-maximin) is the default; we additionally report the
  medoid and borda-rank recommenders, all scored from the
  same run.
\end{itemize}

% ---------------------------------------------------------------------
\subsection{Software and hardware}
\label{app:env}
All experiments were conducted using 2 CPU cores and 8 GiB of RAM. The HPC nodes utilized for the computations are equipped with two AMD Milan 7763 processors and a total of 256 GiB of main memory.

\section{Additional Results}

% ---------------------------------------------------------------------
\subsection{Significance tests}
\label{app:sig}

We test all methods that receive the same transferred candidate pool: \our,
CV-UCB, CV-TS, $\rho\text{-KG}^{\mathrm{apx}}$ and BMA-UCB on the four
benchmarks the experiment section reports: the synthetic spike testbed and the
three real-world applications. Because every method is run from the same initial
designs and the same GP conditioning on each seed, runs are matched across methods
and every comparison is paired. We summarize each run by two scalars: the
final CVaR regret of the design recommended at the last budget, and the
anytime CVaR regret, i.e. the regret of the recommended design averaged
over the BO phase, which summarizes the whole curve the figures show rather than a
single endpoint. On each pair of methods we run the two-sided exact Wilcoxon
signed-rank test on the per-seed values at $\alpha = 0.05$; each pair is tested on
its own, without a multiplicity correction. All tests use seeds $0$--$9$, and \our
is scored by the $\Gamma$-maximin read-out used in the main text. \our is the only method that is never significantly outperformed
under either metric, and it accumulates the most significant wins under both
($4$ on final regret, $6$ on anytime regret). BMA-UCB is the
strongest baseline overall and is the only one besides \our to win any comparison,
though it also takes a loss.

% generated by _sig_win_tie_loss.py --metric both --family transfer --scope paper
%              --rule pessimistic --cap 10
\begin{table}[t]
\centering
\small
\begin{tabular}{lcc}
\toprule
Method & final W/T/L & AUC W/T/L \\
\midrule
\textbf{\our} & 4/12/0 & 6/10/0 \\
$\rho$KG$^{\mathrm{apx}}$ & 1/14/1 & 0/15/1 \\
CV-TS & 0/15/1 & 0/14/2 \\
BMA-UCB & 0/15/1 & 2/13/1 \\
CV-UCB & 0/14/2 & 0/12/4 \\
\bottomrule
\end{tabular}
\caption{Significant wins/ties/losses over every pair of methods that share the
transferred candidate pool, pooled over the four benchmarks (spike, robot pushing, reactor and SPEAR. Each entry is a paired two-sided exact Wilcoxon signed-rank test over
the shared seeds: a win denotes significantly lower regret ($p<0.05$), a loss
significantly higher, and a tie no significant difference.}
\label{tab:win_tie_loss}
\end{table}

% ---------------------------------------------------------------------
\subsection{Robustness to the recommender rule}
\label{app:add_recommender}
\begin{figure}[t]
\centering
\includegraphics[width=0.7\linewidth]{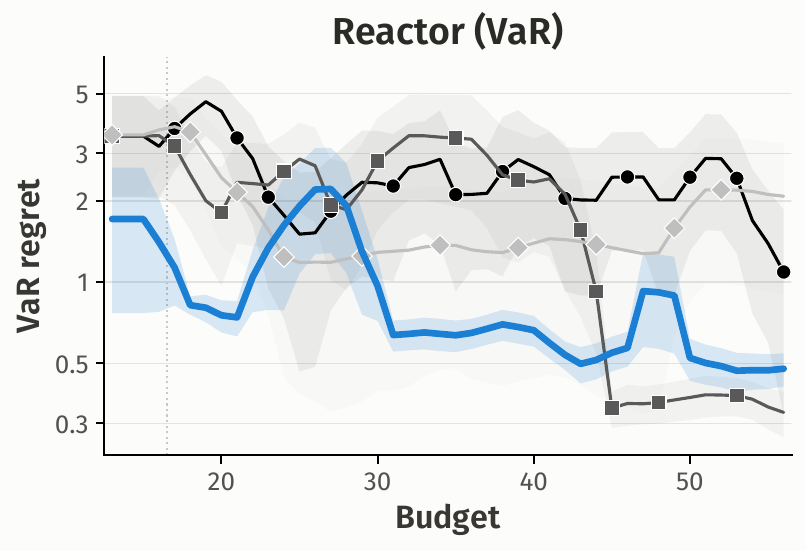}
\caption{Reactor under $\mathrm{VaR}_{0.98}$: VaR regret vs.\ budget. The method
ordering matches the CVaR result. Mean $\pm$ 1 SEM.}
\label{fig:add_var}
\end{figure}

The main text reports \our with the pessimistic recommender. Since the
recommendation is a post-hoc read-out of the surviving pool $\Theta_T$ that does
not affect which designs are queried, we can re-score every run under three
rules from identical data. Each rule scores a queried design $x$ by aggregating
its posterior-mean CVaR $m_\theta(x) := \mathrm{CVaR}_\alpha[\mu_T^\theta(x,W)]$
across the surviving candidates $\theta \in \Theta_T$, and returns the queried
design that scores best:
\begin{itemize}
\item $\Gamma$-maximin:
  $\hat x \in \arg\min_{x} \max_{\theta \in \Theta_T} m_\theta(x)$,
  the design with the lowest \emph{worst-case} CVaR over survivors.
\item Medoid: return the queried design closest to the
  survivor-averaged prediction, i.e.\ minimizing
  $\sum_{\theta \in \Theta_T} \big(m_\theta(x) - \bar m(x)\big)^2$ with
  $\bar m(x) = \tfrac{1}{|\Theta_T|}\sum_{\theta} m_\theta(x)$; this favors
  designs on which the survivors agree.
\item Borda-rank: for each survivor $\theta$, rank the queried designs
  by $m_\theta(\cdot)$; return the design with the best total (Borda) rank
  summed over survivors, a robust ordinal aggregate insensitive to the scale of
  any single candidate.
\end{itemize}
Figure~\ref{fig:add_recommender} shows the three real-world tasks under each
rule. The conclusion is unchanged in every panel: \our matches the baselines on
robot push and reaches lower CVaR regret on the reactor and SPEAR regardless of
which rule is used, so the advantage is not a byproduct of a particular
read-out.

% ---------------------------------------------------------------------

\begin{figure}[t]
    \centering
    \includegraphics[width=0.95\linewidth]{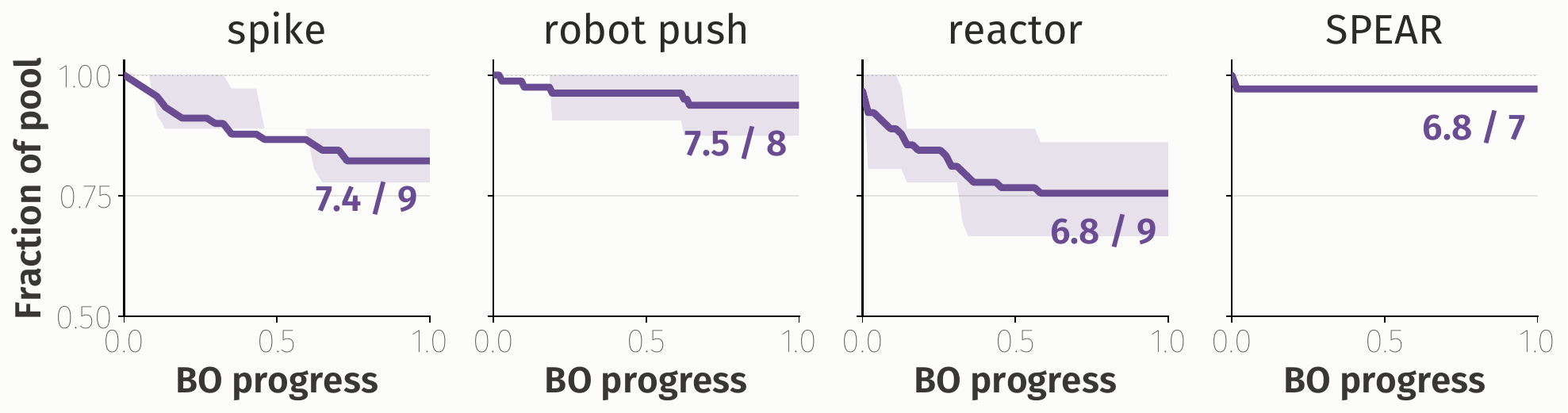}
\caption{Fraction of candidate surrogates surviving over the BO run. On spike and
    the reactor, \our{} prunes source surrogates; on robot push and SPEAR the pool
    remains essentially intact.}
    \label{fig:ablation_elim}
\end{figure}
\begin{figure}[t]
\centering
\includegraphics[width=0.48\linewidth]{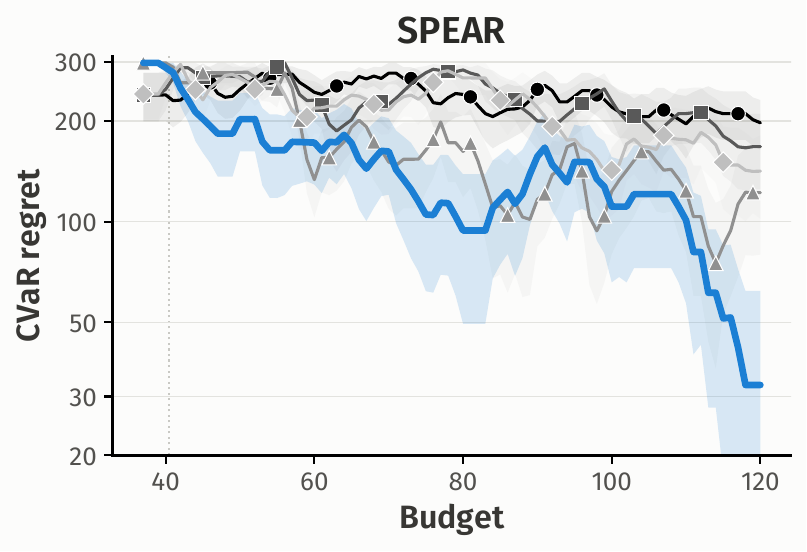}\hfill
\includegraphics[width=0.48\linewidth]{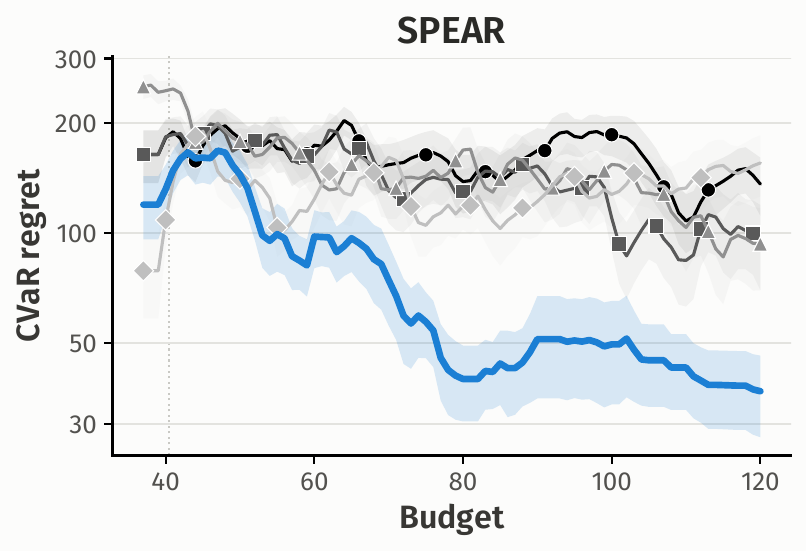}\hfill
\includegraphics[width=0.48\linewidth]{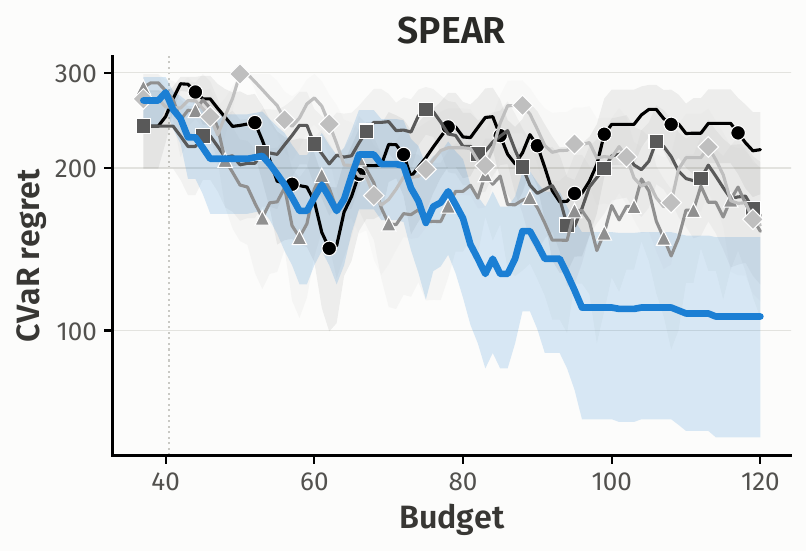}\hfill
\includegraphics[width=0.48\linewidth]{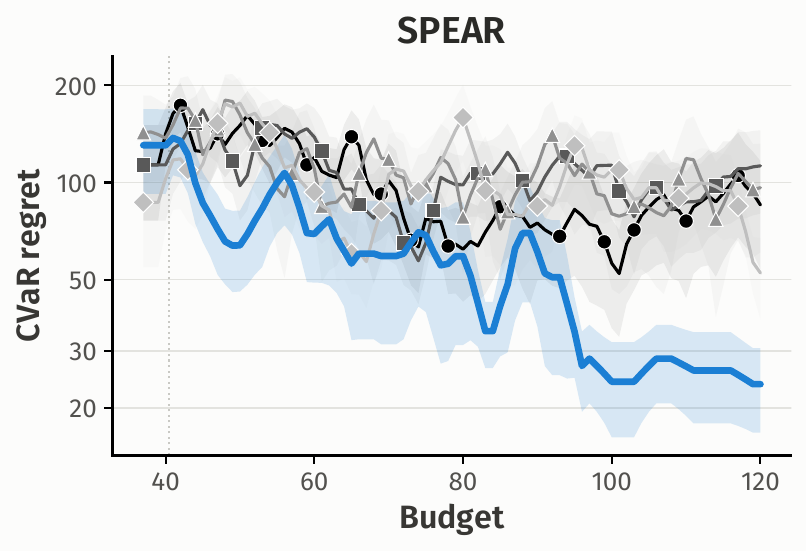}
\caption{SPEAR across targets and pool budgets: $t_1$/$t_5$ at snapshot budgets
$120$ and $300$ (left to right). Mean
$\pm$ 1 SEM.}
\label{fig:add_spear}
\end{figure}

% ------------------------------% ---------------------------------------------------------------------
\subsection{Ablation over targets and pool budgets (SPEAR)}
\label{app:add_spear}
The SPEAR setup involves two design choices whose effect on the result is worth
isolating: which cluster is held out as the target distribution over $W$, and the
snapshot budget at which the source experts are frozen. We ablate both.
Figure~\ref{fig:add_spear} reports two target clusters ($t_1$, $t_5$) crossed with
two snapshot budgets ($120$, $300$). \our attains the lowest final CVaR regret in
all four combinations, and the margin is stable across them: neither the choice of
target nor the amount of source data changes the qualitative picture. The
single-selection baselines periodically commit to a source that times out on the
held-out instances, precisely the failure \our avoids by hedging over the
surviving pool rather than committing to one member.

 \subsection{Elimination dynamics}
\label{app:add_elim}
Figure~\ref{fig:ablation_elim} tracks the fraction of the candidate pool that
survives over the BO run. Elimination is deliberately conservative: a candidate
is removed only once its accumulated residuals are statistically inconsistent
with the data, so a correctly specified surrogate is never discarded with high probability. As the main
text notes, the guarantee requires only that \emph{some} correct candidate
survive; additional survivors add conservatism but do not invalidate it. We
therefore expect and observe that elimination fires only where the pool
contains members the data can rule out: it prunes the confidently tail-blind
sources on the spike and the reactor, and stays essentially inert on SPEAR and
robot push, where every candidate is already a reasonable model of the target.

Low elimination is thus the predicted behavior on a well-specified pool, not a
degenerate case, because \our{} rests on two complementary mechanisms. When
candidates are separable, elimination prunes the confidently-wrong ones. When
they are not (as on SPEAR) elimination correctly abstains and $\Theta_T \approx
\Theta$, and the second mechanism takes over: a single-selection baseline must
still commit to one surrogate, and marginal likelihood cannot tell which of
several equally plausible candidates best captures the tail, so whichever it picks
occasionally lands on a source that times out on the held-out instances. \our{}
instead returns the design that is best in the worst case over \emph{all}
surviving candidates, so it never places that bet. When elimination is inert,
\our{} reduces to this robust recommendation over the transferred pool---losing
nothing, since no candidate is statistically ruled out---while on the spike and
reactor elimination sharpens it further by removing members that would otherwise
enter the worst-case recommendation.
% ---------------- 

\begin{figure*}[t]
\centering
\includegraphics[width=0.9\linewidth]{figures/realworld/push_goal_legend.pdf}\\[2pt]
\emph{Medoid}\\
\includegraphics[width=0.30\linewidth]{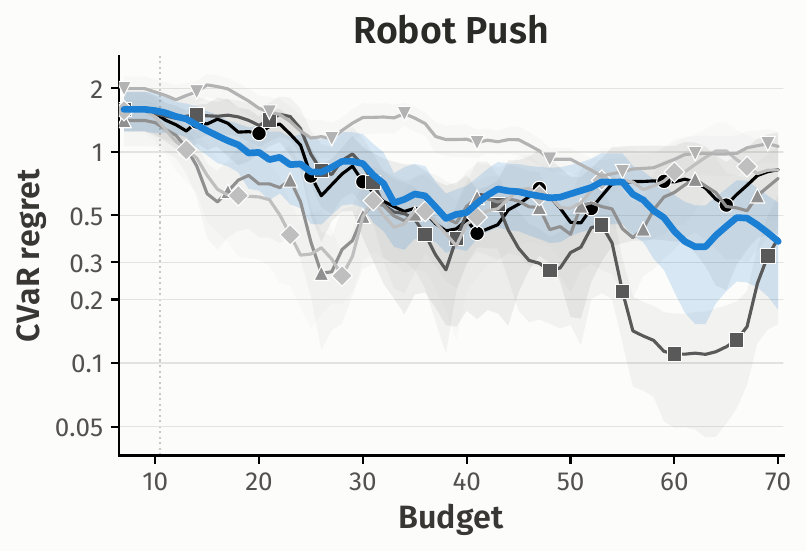}\hfill
\includegraphics[width=0.30\linewidth]{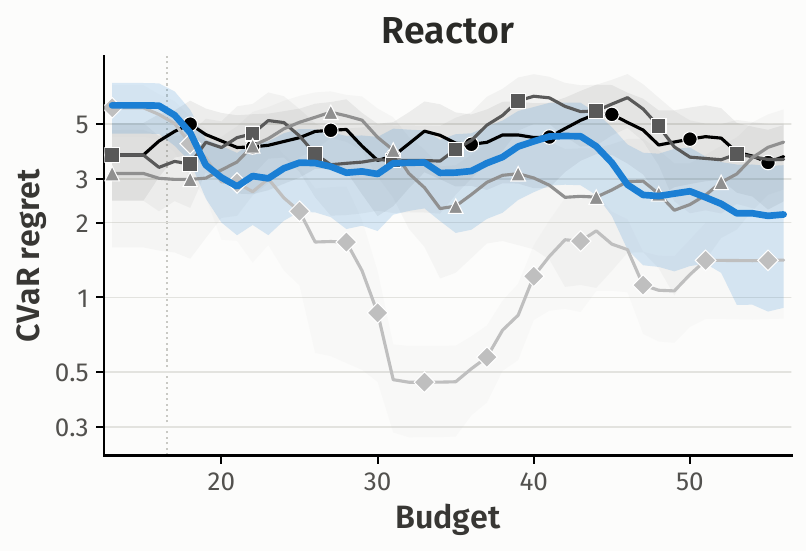}\hfill
\includegraphics[width=0.30\linewidth]{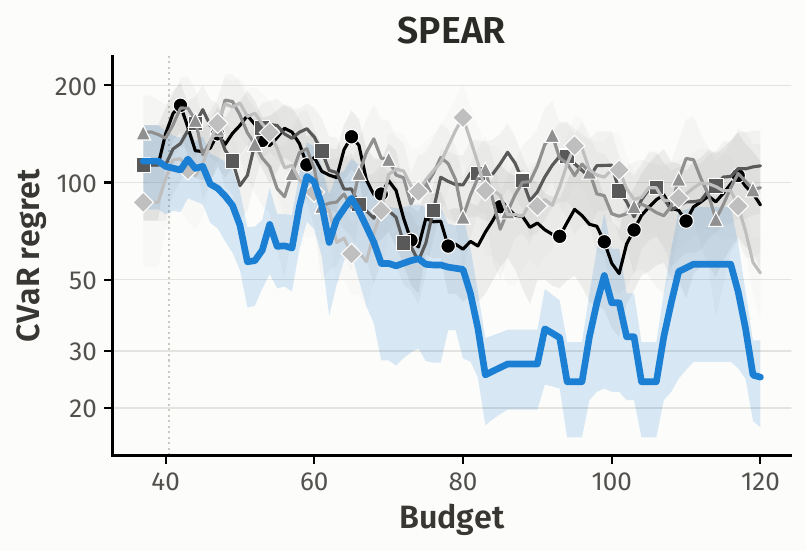}\\[2pt]
\emph{$\Gamma$-maximin}\\
\includegraphics[width=0.30\linewidth]{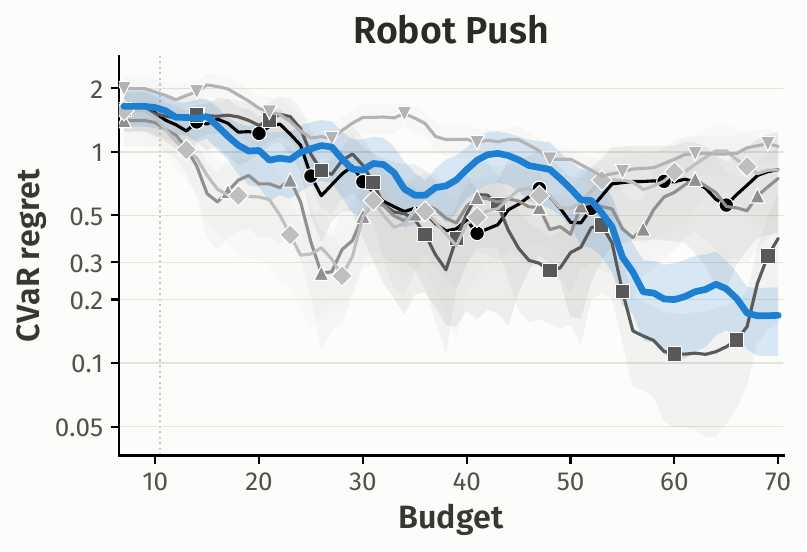}\hfill
\includegraphics[width=0.30\linewidth]{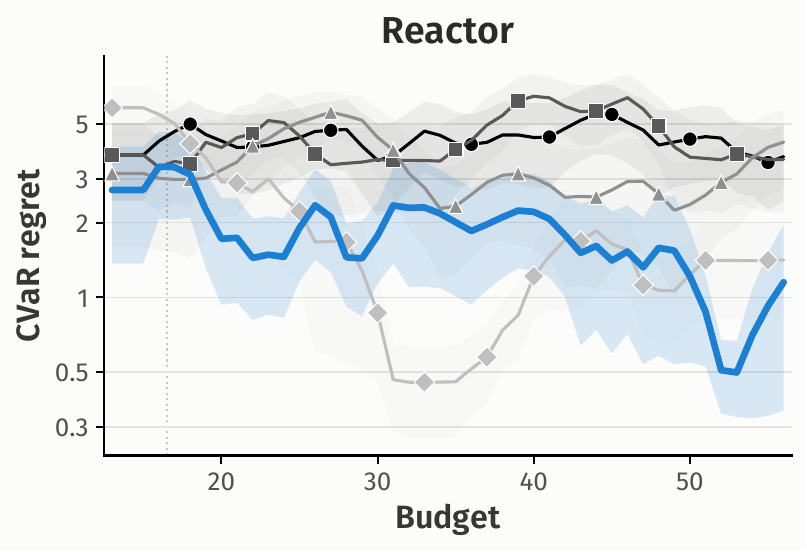}\hfill
\includegraphics[width=0.30\linewidth]{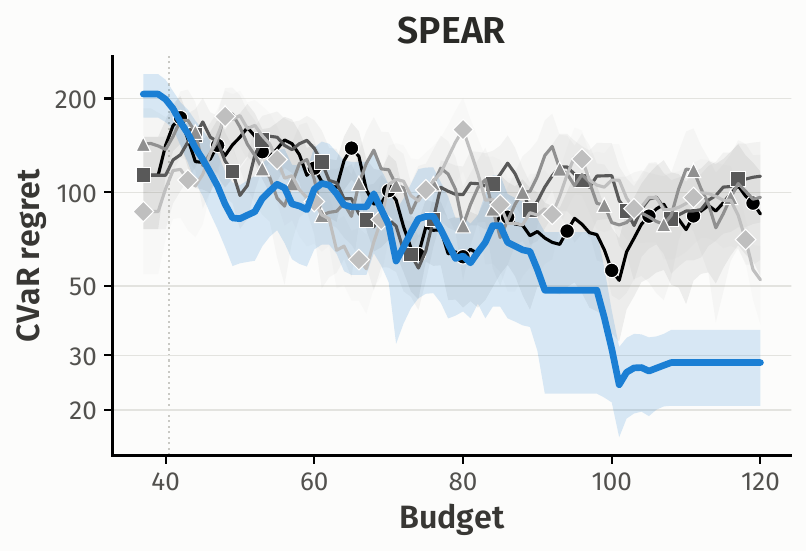}\\[2pt]
\emph{Borda-rank}\\
\includegraphics[width=0.30\linewidth]{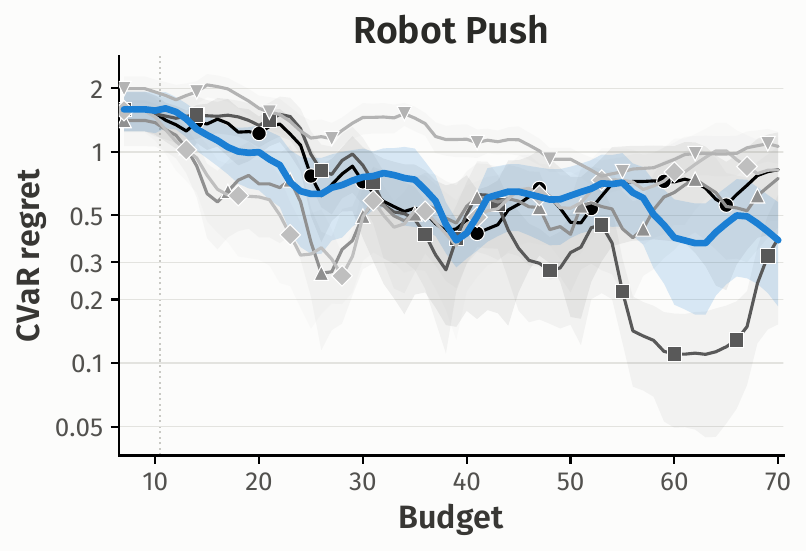}\hfill
\includegraphics[width=0.30\linewidth]{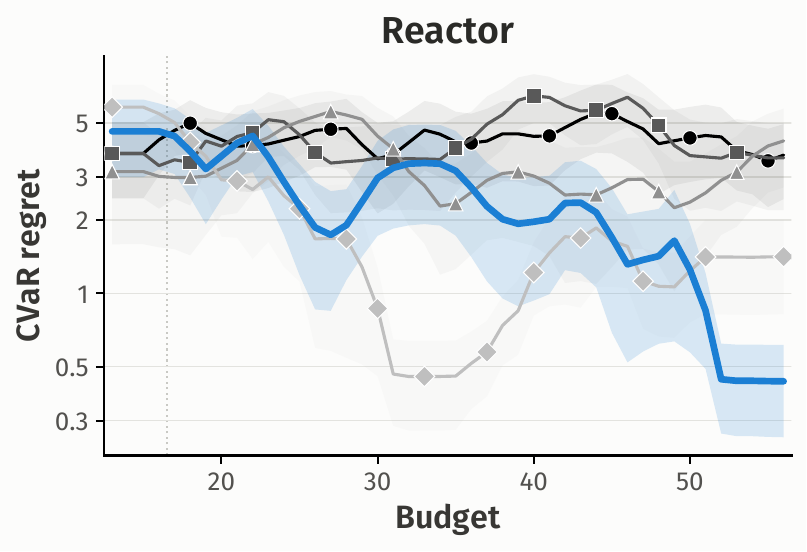}\hfill
\includegraphics[width=0.30\linewidth]{figures/appendix/borda/spear_t5_bo300k3_mean_smooth.pdf}
\caption{CVaR regret vs.\ budget on robot push (left), reactor (middle), and SPEAR
(right) under the three recommender rules (rows).}
\label{fig:add_recommender}
\end{figure*}

% ---------------------------------------------------------------------
\subsection{An alternative risk measure: VaR}
\label{app:add_var}
Value-at-Risk is the other standard tail summary alongside CVaR. At level
$\alpha$, $\mathrm{VaR}_\alpha[Z]$ is simply the $\alpha$-quantile of the outcome
$Z = f(x,W)$: the threshold that the loss stays below with probability $\alpha$.
$\mathrm{CVaR}_\alpha[Z]$ is the average of the worst $(1-\alpha)$ fraction of
outcomes, i.e.\ the mean of $Z$ conditional on exceeding
$\mathrm{VaR}_\alpha[Z]$. The two therefore answer different questions: VaR asks
how bad is the boundary of the tail, while CVaR asks how bad is the
tail on average. This makes VaR blind to the shape of the tail beyond the
quantile, two designs with the same $\mathrm{VaR}_\alpha$ can have very different
worst-case behavior, whereas CVaR integrates over it and is sensitive to rare,
extreme losses. CVaR is also coherent (in particular subadditive), a property VaR
lacks, which is why CVaR is usually preferred in risk-averse optimization.
Nonetheless VaR remains widely used, and reporting it here lets us check that our
gains are not an artifact of the specific functional we optimize.

Importantly, nothing in \our{} is tied to CVaR. The pointwise confidence bands,
the elimination criteria, and the robust recommendation interact with the risk
functional only through a monotone, translation-invariant read-out of the
context distribution, the two properties the CVaR-shift argument (Lemma~1)
actually uses. VaR satisfies both, so the entire machinery applies to it
unchanged, simply by replacing every $\mathrm{CVaR}_\alpha[\cdot]$ with
$\mathrm{VaR}_\alpha[\cdot]$ in the bands and the acquisition rule.

Figure~\ref{fig:add_var} re-runs the reactor experiment with the objective and the
bands re-instantiated at $\mathrm{VaR}_{0.98}$ in place of $\mathrm{CVaR}_{0.98}$;
all other settings are held fixed. \our again separates from the baselines and
attains the lowest final VaR regret, confirming that its advantage comes from the
credal treatment of the source pool---eliminating tail-misjudging members and
hedging over the survivors---rather than from any property specific to the CVaR
functional. We note that this concerns the \emph{algorithm}; the regret guarantee
additionally relies on a lacing-value construction (Step~2 of the proof), which we
establish for CVaR and which admits a VaR analogue via the V-UCB analysis of
\citet{nguyenquoc-icml21a}, but a unified guarantee across risk measures is left
to future work.

% \newpage
%\input{ReproducibilityChecklist.tex}
\end{document}